\documentclass[lettersize,journal]{IEEEtran}
\usepackage{amsmath,amsfonts}
\usepackage{algorithmic}
\usepackage{array}
\usepackage{hyperref}
\usepackage{caption}
\usepackage{subfig}
\usepackage{textcomp}
\usepackage{stfloats}
\usepackage{url}
\usepackage{verbatim}
\usepackage{graphicx}
\usepackage{hyperref}
\usepackage{array}
\usepackage{multirow}
\usepackage{amsmath}
\usepackage{amsthm}
\usepackage{amssymb}
\usepackage{etoolbox}
\usepackage{algorithmic}
\usepackage{algorithm}
\usepackage{booktabs} 
\usepackage{graphicx} 
\usepackage{multirow} 
\usepackage{array} 
\usepackage{xspace}
\usepackage{newtxtext,newtxmath}
\usepackage{balance}
\usepackage{cite}
\allowdisplaybreaks
\newcommand{\ie}{i.e.,\xspace}

\newtheorem{myTheo}{Theorem}
\newtheorem{myAssu}{Assumption}
\newtheorem{myLem}{Lemma}

\newtheorem{myPro}{Proposition}

\def\BibTeX{{\rm B\kern-.05em{\sc i\kern-.025em b}\kern-.08em
		T\kern-.1667em\lower.7ex\hbox{E}\kern-.125emX}}
\begin{document}
	\title{
    Deploying Large AI Models on Resource-Limited Devices with Split Federated Learning}
	\author{
		Xianke Qiang, Hongda Liu, Xinran Zhang, Zheng Chang,~\IEEEmembership{Senior~Member,~IEEE},Ying-Chang Liang,~\IEEEmembership{Fellow,~IEEE}\newline
		GitHub: \url{https://github.com/XiankeQiang/SFLAM}
		\thanks{X. Qiang, X. Zhang and Z. Chang are with School of Computer Science and Engineering, University of Electronic Science and Technology of China, Chengdu 611731, China. H. Liu is with the School of Electronics and Communication Engineering, the Shenzhen Campus of Sun Yat-sen University, Sun Yat-sen University, Shenzhen, 518107, China. Y-C. Liang is with Center for Intelligent Networking and Communications (CINC), University of Electronic Science and Technology of China, 611731 Chengdu, China.
}
	}
	
	\maketitle
	
	\begin{abstract}
		Large Artificial Intelligence Models (LAMs) powered by massive datasets, extensive parameter scales, and extensive computational resources, leading to significant transformations across various industries. Yet, their practical deployment on resource-limited mobile edge devices is hindered by critical challenges such as data privacy, constrained resources, and high overhead costs. Addressing this gap, this paper proposes a novel framework, named Quantized Split Federated Fine-Tuning Large AI Model (SFLAM). By partitioning the training load between edge devices and servers using a split learning paradigm, SFLAM can facilitate the operation of large models on devices and significantly lowers the memory requirements on edge devices.
		Additionally, SFLAM incorporates quantization management, power control, and bandwidth allocation strategies to enhance training efficiency while concurrently reducing energy consumption and communication latency. A theoretical analysis exploring the latency-energy trade-off is presented, and the framework's efficacy is validated via comprehensive simulations. The findings indicate that SFLAM achieves superior performance in terms of learning efficiency and scalability compared to conventional methods, thereby providing a valuable approach for enabling advanced AI services in resource-constrained scenarios.
		
	\end{abstract}


	\section{Introduction}

    The advent of Large AI Models (LAMs), such as ChatGPT and DeepSeek, marked a significant leap in AI capabilities, powered by their extensive parameter scales, large-scale datasets, and substantial computational resources \cite{10579546}. These models exhibit exceptional generalization and adaptability, proving crucial across diverse applications. As user demand for ubiquitous AI access and real-time, personalized experiences grows, deploying and training these models on mobile devices becomes increasingly relevant \cite{10.1145/3641512.3686358}. To meet these escalating demands, fine-tuning, which involves adapting pre-trained models with domain-specific data, has become a widely adopted and efficient strategy for enhancing LAM performance on specialized tasks, offering a cost-effective path to superior results. \par

 Fine-tuning large AI models presents two key challenges: data scarcity and resource limitations. While the proliferation of mobile edge devices and advancements in sensing technologies facilitate distributed data collection, privacy concerns, particularly in sensitive domains such as healthcare \cite{thirunavukarasu2023large} and finance \cite{wu2023bloomberggpt}, often restrict direct data sharing. This necessitates collaborative training methodologies that leverage distributed data while safeguarding privacy. Federated Learning (FL) offers a solution by enabling multiple data owners to collaboratively fine-tune LAMs without exchanging raw data \cite{cai2023efficient,dayan2021federated, 10.1145/3637528.3671582}. Nevertheless, the increasing scale of LAMs presents substantial computational and storage demands. Many real-world edge devices lack the requisite processing capacity and memory to support FL, thus restricting its deployment in mobile environments.\par
 
      \begin{figure}[t]
    	\centering
    	\includegraphics[width=0.45\textwidth]{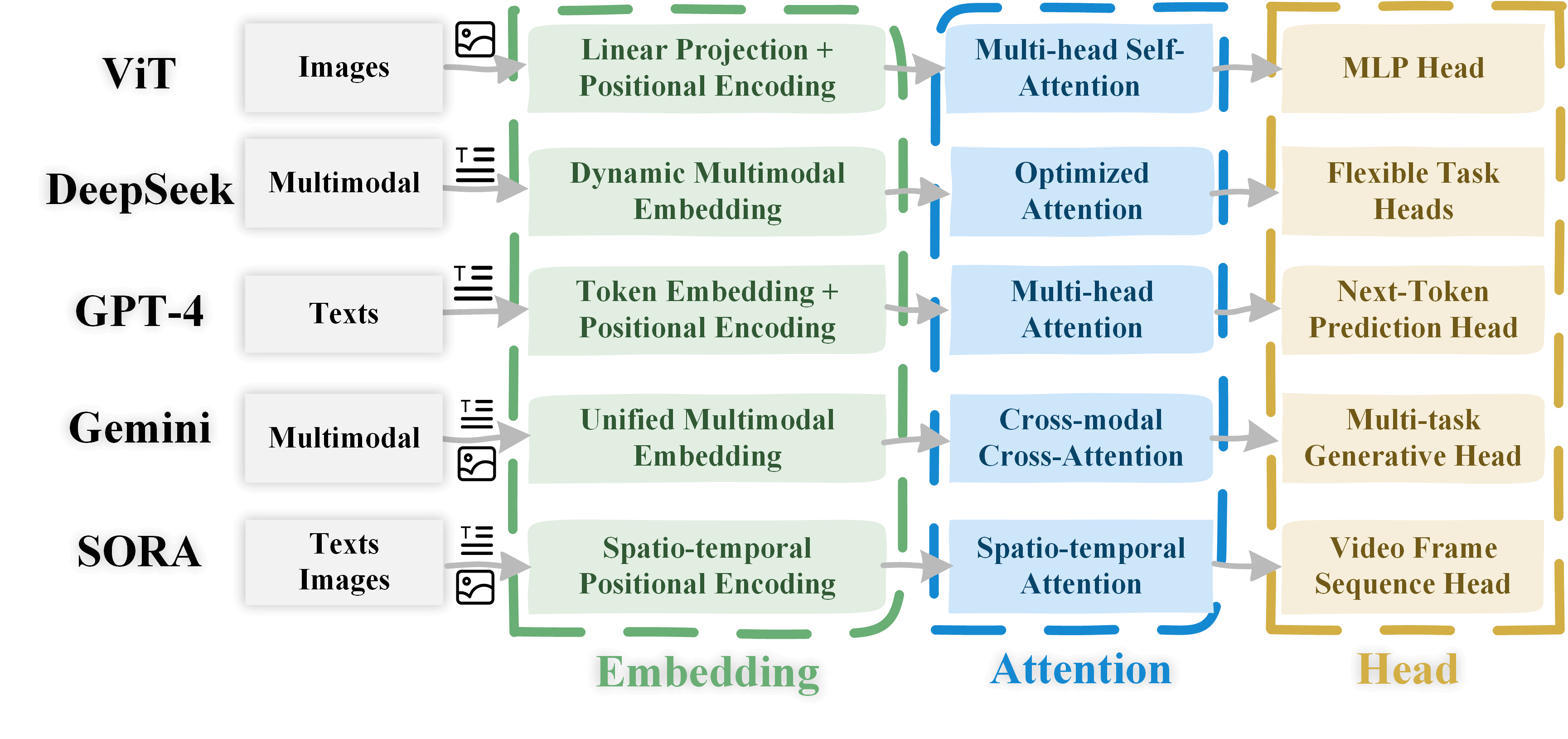}
    	\caption{The architecture of LAMs.}
    	\label{fig1}
    \end{figure}
 Split Federated Learning (SFL) presents a viable approach by partitioning model training between edge devices and a central server \cite{thapa2022splitfed, 10714368, 10839234}. Devices compute only the initial model layers, forwarding intermediate activations to the server for subsequent processing. This strategy significantly reduces device computational load, enabling the deployment of Large AI Models (LAMs) on resource-constrained devices. By integrating SFL with fine-tuning techniques, not only are privacy concerns mitigated, but large models can also be effectively executed on resource-constrained devices. However, SFL introduces communication overhead as a key challenge, due to the substantial bandwidth required for transmitting intermediate activations, device weights, and gradients. \par 

   
 Modern LAM architectures, such as Vision Transformer (ViT), BERT, GPT, Gemini, and SORA, typically consist of three main components: an input-based embedding layer, an attention-based transformer module, and a task-specific output head (see Fig. \ref{fig1}). To mitigate communication bottlenecks, we propose to partition the model at the embedding layer, assigning the embedding layer to the device and the attention and output layers to the server. Taking ViT as an example, different model configurations lead to varying computation and communication demands. As shown in Table \ref{table0}, this approach effectively reduces the device's computation and memory footprint, making it suitable for resource-constrained edge devices (such as smartphones and Raspberry Pi). However, despite  the reduction in computational overhead, communication overhead remains a critical challenge, particularly in large-scale deployments.\par

        \begin{table}[t]
            \centering
            \caption{Comparison of Communication, Computation, and Memory Overhead for Devices Across Different Paradigms}
            \label{table0} 
            \scriptsize
            \begin{tabular}{lcccc}
                \toprule
                \textbf{} & CL & \multicolumn{2}{c}{FL} & SFL \\ 
                \textbf{} & {} & ViT-B/32 & ResNet50 & ViT-B/32 \\ 
                \midrule
                Communication (MB) \texorpdfstring{\textsuperscript{1}}{} & 36.67 & 333.64 & 89.96 & 27.688 \\ 
                Compute (GFLOPs)\texorpdfstring{\textsuperscript{2}}{} & 0 & 558.99 & 530.28 & 14.80 \\ 
                Memory (MB)\texorpdfstring{\textsuperscript{2}}{}  & 0 & 3643.23 & 10430.71 & 18.38 \\ 
                \bottomrule
            \end{tabular}
        \end{table}
        
        \begin{table}[t]
            \centering
            \caption{Comparison of Embedding Layer Characteristics Across ViT Variants}
            \scriptsize
            \begin{tabular}{lcccc}
                \toprule
                & ViT-B/32 & ViT-L/32 & ViT-B/16 & ViT-L/16 \\ 
                \midrule
                Parameters (MB) & 9.0 & 12.0 & 2.25 & 3.0 \\ 
                Compute (GFLOPs) & 14.80 & 19.73 & 14.80 & 19.73\\ 
                Memory (MB) & 18.38 & 24.50 & 74.00 & 98 \\ 
                Output Shape & (b,50,768) & (b,50,1024) & (b,197,768) & (b,197,1024) \\ 
                \bottomrule
            \end{tabular}
            \label{table1}
        \end{table}
    \footnotetext[1]{The communication overhead of different schemes is compared. For a $224 \times 224 \times 3$ image (4 bytes per pixel), the raw image size is $0.573\text{MB}$. After passing through the ViT-B/32 embedding layer, the intermediate data size becomes $0.146 \text{MB}$ per image. For one batch 128 images, transmitting the raw images costs $73.344\text{MB}$, while transmitting the ViT-B/32 model incurs $333.64\text{MB}$, using SFL reduces the overhead to approximately $9+18.688=27.688\text{MB}$.}
    \footnotetext[2]{The computational FLOPs and memory usage were measured using the tools available \href{https://github.com/facebookresearch/fvcore/blob/main/docs/flop_count.md}{here}.}
   
    To address communication bottlenecks, we introduce Quantized Split Federated Fine-Tuning for Large AI Models (SFLAM), a novel framework designed for efficient model training on resource-constrained devices over wireless networks. As shown in Table \ref{table1}, SFLAM enables devices to participate in training large-scale visual models, such as ViT, with memory footprints below 100MB. This advancement facilitates improved scalability and performance for LAM deployment at the network edge. This study pioneers a comprehensive investigation into the potential of quantized split federated fine-tuning. Our key contributions are outlined below:
        \begin{itemize}
            \item This paper introduces SFLAM, an energy-efficient and performance-aware Quantized Split Federated Fine-Tuning framework, designed to enable the deployment of LAMs on resource-constrained edge devices. By offloading some computations and leveraging quantization techniques, SFLAM significantly reduces computational, memory, and energy demands, facilitating effective LAM operation on devices within these limitations. 
            
            \item We conduct a convergence analysis of SFLAM and formulate an efficiency-driven multi-objective optimization problem, integrating resource allocation, quantization management, power control, and bandwidth allocation.
            
            \item The formulated optimization problem is a mixed-integer non-linear and non-convex problem, which is NP-hard and challenging to solve directly. To address this, we decompose the problem into three subproblems—optimal power control, optimal quantization management, and optimal bandwidth allocation—which are solved iteratively using Successive Convex Approximation (SCA), Greedy Algorithm, and the Perfect Matching Method, respectively.
            \item We validate the effectiveness of SFLAM through extensive simulations on multiple benchmark datasets. Compared to existing methods, SFLAM demonstrates superior efficiency and improved learning performance across heterogeneous resource-constrained devices.
        \end{itemize}
        The rest of the paper is organized as follows. In Section \ref{section2}, the related works are introduced. Section \ref{section3} introduces the SFLAM architecture, system model, and latency-energy analysis. Section \ref{section4} analyzes the convergence of SFL with quantized activation. Section \ref{section5} formulates the optimization problem, while Section \ref{section6} derives the optimal power control, quantization management, and bandwidth allocation. Section \ref{section7} presents the simulation results. Finally, Section \ref{section8} concludes the article.
	

	\section{Related Work}
        \label{section2}
            Modern Large AI Models (LAMs) are typically trained through a two-stage process: pre-training and fine-tuning. In the pre-training phase, models are trained on extensive, often publicly available datasets. Subsequently, fine-tuning adapts these pre-trained models to specific tasks using device-resident data. This pretrain-then-finetune paradigm, leveraging large-scale datasets like ImageNet \cite{deng2009imagenet} and OSCAR corpora \cite{abadji2021ungoliant}, has become prevalent for refining LAMs for downstream language and vision tasks, including text/image classification and generation \cite{cai2023efficient}. State-of-the-art (SOTA) LAMs employing this workflow are predominantly transformer-based architectures \cite{vaswani2017attention, bomze2010interior, dosovitskiy2020image}. To mitigate the substantial computational costs associated with fine-tuning these large models, Parameter-Efficient Fine-Tuning (PEFT) techniques, such as Adapter Tuning \cite{houlsby2019parameter}, Prompt Tuning \cite{li2023prompt}, and Low-Rank Adaptation (LoRA) \cite{hu2021lora}, have been developed, significantly reducing resource requirements.\par

            To achieve collaborative training over wireless networks, some work leverages FL to allow multiple users to collaboratively fine-tune LAMs without the need to share data. In these studies, the entire model is deployed on local devices for training, with model parameters or gradients aggregated at each round\cite{10820175, 10542529}. To further enhance efficiency, researchers have explored FL in conjunction with PEFT techniques. FedAdapter \cite{cai2023efficient} dynamically adjusts adapter configurations to accelerate model convergence. PromptFL \cite{10210127} enables federated participants to collaboratively train shared prompts instead of the full model.\par
            Split Learning (SL) offers another promising approach to reduce computational overhead and memory consumption by partitioning models across devices and servers \cite{thapa2022splitfed}. To accelerate SL training, parallel Split Federated Learning (SFL) frameworks have been proposed, enabling simultaneous training of multiple server-side and device-side models \cite{10714368, 10839234, 10.1007/978-3-031-43895-0_33,10947249}. These frameworks, such as ASFV \cite{10714368}, which adaptively splits the model and parallelizes training considering vehicle selection and resource allocation, and those that dynamically adjust the cut layer for parallel execution to optimize performance under non-IID data \cite{10839234}, aim to enhance training efficiency. Additionally, FeSViBS\cite{10.1007/978-3-031-43895-0_33} framework builds upon the existing federated split vision transformer and introduces a block sampling module. A recent U-shaped SFL framework \cite{10947249} employs semantic-aware auto-encoders and reinforcement learning to enhance privacy, communication, and performance in vehicular networks.\par
            
            While these approaches enable collaborative fine-tuning across distributed devices while reducing computational and communication costs, they still demand significant memory resources. Even with SOTA PEFT methods, pre-training and fine-tuning require approximately several GBs of memory \cite{liu2024sparse}, making them impractical for resource-constrained devices such as mobile phones and IoT devices. To address this challenge, we propose an energy-efficient and memory-free split federated learning framework for LAMs, specifically designed for distributed pre-training and fine-tuning. We further provide a theoretical convergence analysis and release our framework as open-source to promote further research in this area.

	\section{Split Federated Learning for LAMs: Architecture and System Model}
        \label{section3}
	In this section, we consider a typical LAMs training scenario over wireless network comprising  a server and multiple devices. Then we introduce the system architecture and SFLAM workflow.
        \begin{figure}[t]
		\centering
		\includegraphics[width=0.30\textwidth]{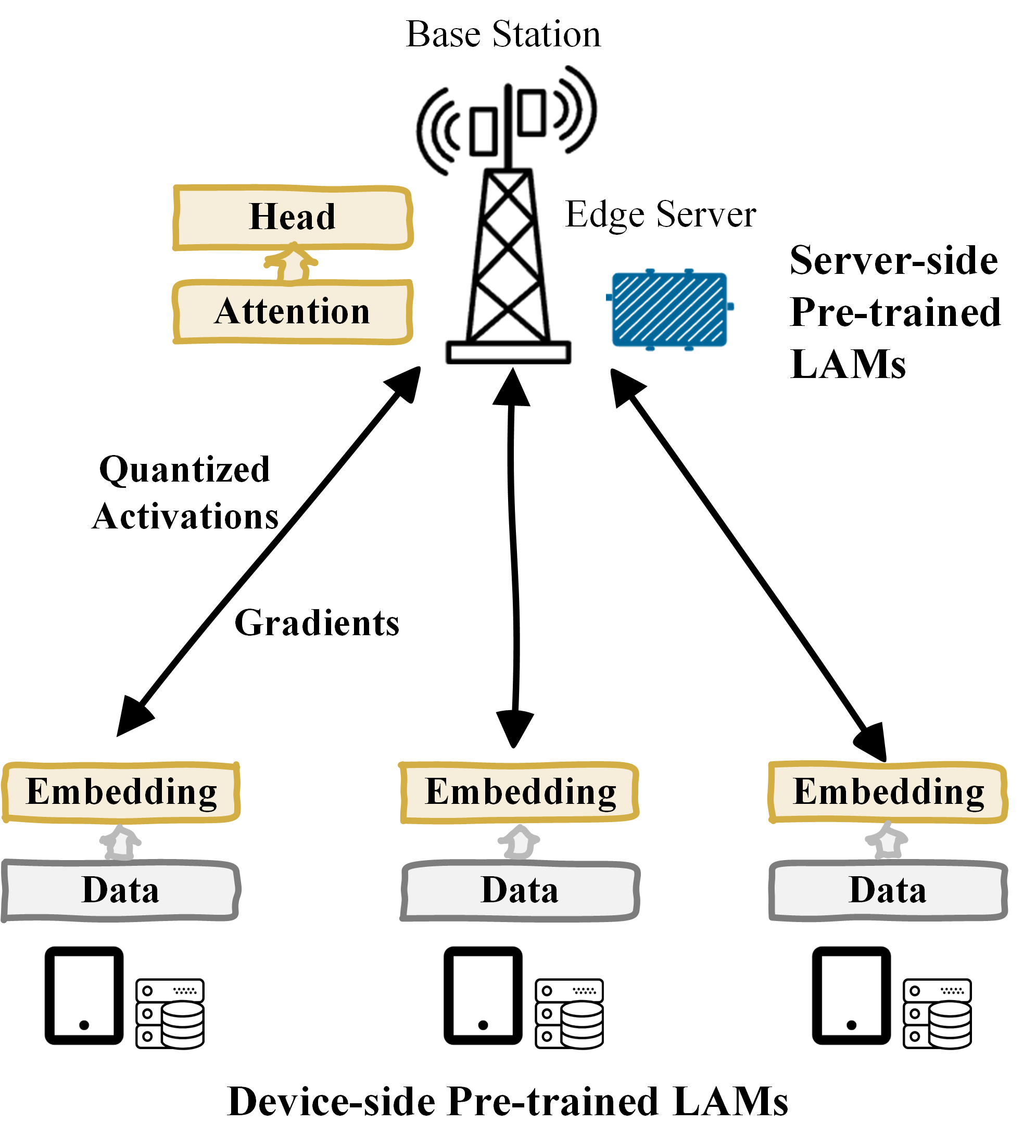}
		\caption{The framework of SFLAM.}
		\label{fig2}
	\end{figure}  
    
    \subsection{Split Federated Pre-training and Fine-tuning Framework}
    As shown in Fig. \ref{fig2}, we consider a typical SFLAM scenario over a wireless network comprising a server and a set of devices $\mathcal{V} = \{1,2, \dots,V\}$. We define the set of the global communication rounds as $\mathcal{K}=\{1,2,\dots,K\}$. Given the constraints of limited wireless resources, only a subset $\mathcal{N}_k \in \mathcal{V}$ of devices can participate in the SFLAM tasks, with $|\mathcal{N}_k| = N$. Each device $n \in \mathcal{N}_k$ collects its private data set $\mathcal{D}_n$ with the data size $D_n = |\mathcal{D}_n|$. In the supervised learning tasks, we have $\mathcal{D}_n=\{\mathcal{X}_n,\mathcal{Y}_n\}$, where $\mathcal{X}_n = \{x_{n,1},x_{n,2},\dots,x_{n,D_n}\}$ represents the training data and $\mathcal{Y}_n = \{y_{n,1},y_{n,2},\dots,y_{n,D_n}\}$ refers to the labels of the device $n$. This data can be utilized for various SFL pre-training and fine-tuning tasks, and all the devices jointly train a global SFL model guided by the server.\par
    
    In SFLAM, the LAM is split into two parts: the embedding layer remains on the device device, while the attention and head layers are handled by the server. As illustrated in Fig. \ref{fig2}, the training process consists of four steps. First, the device computes the forward pass of the embedding layer and transmits the quantized activations to the server. The server then processes these activations through its model and returns the corresponding gradients. Next, the device updates the embedding layer weights via backpropagation. Finally, the server aggregates updates from both sides to refine the global model.\par
    
    This SFLAM structure enables privacy-preserving patch embedding on the device, while offloading the complex transformer-based attention and head layers to the server. We decompose the whole transformer model as two parts in terms of the embedding layer $\boldsymbol{\omega}_n^\mathcal{C}$ of device $n$ as well as transformer-based attention and head layer $\boldsymbol{\omega}^\mathcal{S}$. Let $\boldsymbol{\omega} = \{\boldsymbol{\omega}_n^\mathcal{C},\boldsymbol{\omega}_n^\mathcal{S}\}\in \mathbb{R}^{\boldsymbol{d}}$ represents the whole model and $\boldsymbol{\ell}(\boldsymbol{\omega},x_n^{i})$ denotes the loss function of data sample $i$. For each local dataset $\mathcal{D}_n$, the local loss function is defined as $\boldsymbol{L}_n(\boldsymbol{\omega}) = \frac{1}{\left|\mathcal{D}_n\right|} \sum_{i=1}^{\left|\mathcal{D}_n\right|}\boldsymbol{\ell}(\boldsymbol{\omega},x_{n}^i)$. The objective of the training process is to minimize the global loss, which is defined as:
	\begin{equation}
		\min_{\boldsymbol{\omega}} \mathbf{L}(\boldsymbol{\omega})=\sum_{n \in \mathcal{V}} \rho_n a_n \boldsymbol{L}_n(\boldsymbol{\omega}),
	\end{equation}
	where $\rho_n = \frac{D_n}{\sum_{n \in \mathcal{N}_k}D_n}$. We define the local epochs as $\mathcal{I} = \{1,2,\cdots,I\}$. Then, by employing the local stochastic gradient descent (SGD), the device $n$ updates its device-side model parameters as
	\begin{align}
		\boldsymbol{\omega}_{n,k+1}^{\mathcal{C},0}&=\boldsymbol{\omega}_{k}^{\mathcal{C}}\\
		\boldsymbol{\omega}_{n,k+1}^{\mathcal{C},i} &= \boldsymbol{\omega}_{n,k+1}^{\mathcal{C},i} - \eta_k \nabla L_n(\boldsymbol{\omega}_{n,k+1}^{\mathcal{C},i-1}),
	\end{align}
	where $\eta_k$ is the learning rate and $\boldsymbol{\omega}_{k+1}^{\mathcal{C}}=\sum_{n \in \mathcal{N}_k} \rho_n  \boldsymbol{\omega}_{n,k+1}^{\mathcal{C},I}$. The server updates its model as
	\begin{align}
		\boldsymbol{\omega}_{n,k+1}^{\mathcal{S},0}&=\boldsymbol{\omega}_{k}^{\mathcal{S}}\\
		\boldsymbol{\omega}_{n,k+1}^{\mathcal{S},i} &= \boldsymbol{\omega}_{n,k+1}^{\mathcal{S},i} - \eta_k \nabla L_n(\boldsymbol{\omega}_{n,k+1}^{\mathcal{S},i-1}).
	\end{align}
	where the  $\boldsymbol{\omega}_{k+1}^{\mathcal{S}}=\sum_{n \in \mathcal{N}_k} \rho_n  \boldsymbol{\omega}_{n,k+1}^{\mathcal{S},I}$. The global model parameter for any iteration $k+1$ is defined as: $\boldsymbol{\omega}_{k+1}=\{\boldsymbol{\omega}_{k+1}^\mathcal{C}, \boldsymbol{\omega}_{k+1}^\mathcal{S}\}$. Then, by aggregating all the local weight differentials, we can obtain the global model parameters $\boldsymbol{\omega}_{k+1}$ by
	\begin{equation}
		\boldsymbol{\omega}_{k+1} = \sum_{n\in\mathcal{N}_k}\rho_{n,k} \boldsymbol{\omega}_{n,k},
	\end{equation}
	where $\mathcal{N}_k = \{n \mid a_{n,k} = 1, \forall n \in \mathcal{N} \}$, with $a_{n,k}$ representing the participation level (or probability) of device $n$. If $a_{n,k} = 1$, device $n$ is available in the $k$-th round. If $a_{n,k} = 0$, device $n$ is unavailable in the $k$-th round.

	
	\subsection{Quantization Model}
	The embedding layer activations, as highlighted in Fig. \ref{fig2}, introduce substantial overhead during SFLAM uplink communication. To achieve lower transmission latency, we quantize the intermediate communication data, including the embedding layer output vector A, effectively compressing it into a smaller representation.\par
	For a $\boldsymbol{d}$-dimensional intermediate communication vector $\mathcal{A}_{n,k}$, we denote $\mathcal{A}_{n,k}^{max}$ and $\mathcal{A}_{n,k}^{min}$ as the upper and lower bounds of the absolute values of embedding output vector. $\mathcal{A}_{n,k,d}(d \in \{1,2,\dots,|\boldsymbol{d}|\})$, such that $\mathcal{A}_{n,k}^{max}=\max\{|\mathcal{A}_{n,k,d}|\}$ and $\mathcal{A}_{n,k}^{min}=\min\{|\mathcal{A}_{n,k,d}|\}$. Then we have $\mathcal{A}_{n,k}^{min}\leq |\mathcal{A}_{n,k,d}| \leq \mathcal{A}_{n,k}^{max}$. We employ stochastic rounding in the quantization method, denoting $q_{n,k}$ as the quantization bits for the device $n$ at global round $k$. Then, there are $2^{q_{n,k}}$ integer values, and $[\mathcal{A}_{n,k}^{min},\mathcal{A}_{n,k}^{max}]$ can be divide into $2^{q_{n,k}}-1$ intervals, denoted as $\mathcal{X}_\phi = [\chi_\phi,\chi_{\phi+1}]$, where $\phi \in \{0,1,2,...,2^{q_{n,k}}-2\}$. The quantization function can be written as 
	\begin{equation}
		\mathcal{Q}(\mathcal{A}_{n,k}) = 
		\begin{cases}
			\text{sign}(\mathcal{Q}(\mathcal{A}_{n,k,d}))\cdot \chi_\phi & \text{w.p. } \iota\\
			\text{sign}(\mathcal{Q}(\mathcal{A}_{n,k,d}))\cdot \chi_{\phi+1} & \text{w.p. } (1-\iota)
		\end{cases}
	\end{equation}
	where we have $\iota = (\chi_{\phi+1} - \mathcal{A}_{n,k,d}/\chi_{\phi+1} - \chi_{\phi})$ and the sign function sign($\cdot$) yields the sign of $\mathcal{A}_{n,k,d}$ as $-1$ or $1$. Moreover, "w.p." is short for "with probability". Then, the quantized intermediate communication vector $\mathcal{Q}(\mathcal{A}_{n,k})=[\mathcal{Q}(\mathcal{A}_{n,k,1}),\mathcal{Q}(\mathcal{A}_{n,k,2}),\cdots,\mathcal{Q}(\mathcal{A}_{n,k,d})]$ is expressed by a total number of 
	\begin{equation}
		s_{n,k}= q_{n,k}|\boldsymbol{d}| + \mu \text{ bits}, 
		\label{num_bits}
	\end{equation}
	where $\mu$ is the sum of the number of bits to represent the sign, and the size of $\mathcal{A}_{n,k}^{min}$ and $\mathcal{A}_{n,k}^{max}$ \cite{9611373}. Utilizing the quantization function defined above, we can obtain the following lemma.
	\begin{myLem}(Unbiased Quantization Scheme \cite{li2017training,10038639})
    A randomized mapping $\mathcal{Q}: \mathbb{R}^d \rightarrow \mathbb{R}^d$ is an unbiased quantization scheme if there exists $\delta$ such that $\mathbb{E}\left[\mathcal{Q}(\mathcal{A})\right] = \mathcal{A}$, $\mathbb{E}\left[\|\mathcal{Q}(\mathcal{A})-\mathcal{A}\|_2^2\right]\leq \delta \|\mathcal{A}\|_2^2$.\par
		\label{lemma1}
	\end{myLem} 
     In general, each element in $\mathcal{A} \in \mathbb{R}^d$ will be independently quantized between $\arg \min \mathcal{A}$ and $\arg \max \mathcal{A}$ with the data precision of $q \leq 32$ bits. Let $q$ be the quantization strategy of activation quantization. The $\delta$ can be defined as $\delta = (1+\sqrt{2d-1})/(2(2^q-1))$, which bounds the quantization error in terms of the squared $L_2$-norm of the quantization error.
	
	\subsection{Latency-Energy Analysis}
    	\subsubsection{Device Computation}
    	Let $\gamma_d$ denote the computational workload (in FLOPs) required for training the device-side model (Embedding Layer) on a single data sample. Since the device-side model processes a mini-batch of data samples, the total computational workload is $b \gamma_d$, where $b$ is the mini-batch size. When device $n$ trains the device-side model locally, the computation delay can be expressed as:
    	\begin{equation}
    		T_n^{cmp} = \frac{b \gamma_d}{f_n C_n D_n},
    	\end{equation}
    \noindent where $f_n$ is the GPU frequency of user n, $C_n$ is the number of cores of the GPU at device $n$ and $D_n^U$ is the number of FLOPs per cycle per core of the GPU. The relationship between the GPU's power consumption and its clock speed is cubic, \ie $power= \kappa_1 f_n^3$. Here, $\kappa_1$ is the coefficient reflecting the power usage per cycle per second cubed (in Watt/(cycle/s)$^3$), depending on the specific GPU architecture \cite{10.1145/3641512.3686358}. Hence, when training one transformer layer, the energy expenditure for local computations is established as follows:
    	\begin{equation}
    		E_n^{cmp} = \kappa_1 f_n^3 \times T_n^{cmp} = \frac{\kappa_1 f_n^2 \gamma_d}{C_n D_n}.
    	\end{equation}
    	\subsubsection{Server Computation}
    	Let $\gamma_s$ denote the computation workload (in Flops) of server-side model's (Attention Layer and Head Layer) training process for processing a data sample. The server-side model execution needs to process a mini-batch of data samples, and the overall computation workload is $b \gamma_s$. When server trains server-side model, the time taken for the computation can be expressed as follows:
    	\begin{equation}
    		T_s^{cmp} = \frac{b \gamma_s}{f_sC_sD_s},
    	\end{equation}
    	\noindent where $f_s$ denotes the GPU frequency of server, and $C_s$ represents the total core count of the GPU of server, and $D_s$ signifies the computational capability oif each core, measured in floating point operations per cycle, for the GPU located at server. The energy used for server-side model training and downlink transmission from the server to devices is not considered, due to server has sufficient energy compared with devices.
    	
    	\subsubsection{Device-Server Communication} 
    	The communication stage contains uplink and downlink stages. The download time can be neglected compared to the upload time, as assumed in \cite{10542529,9611373,10038639}, because the downlink transmit power of the RSU is considerably greater than the uplink powers of vehicles and higher downlink bandwidth is occupied by the RSU for data distribution. So, in this paper, we only consider the uplink communication stage.
    	In this work, we consider the orthogonal frequency division multiple access (OFDMA) with $M$ RBs indexed by $\mathcal{M}=\{1,2,...,M\}$ for devices to upload intermediate parameters. Each device can occupy one uplink RB in a communication round to upload its activation. Let ${\boldsymbol{z}}_{k,t}=({\boldsymbol{z}}_{k,t}^{(1)},{\boldsymbol{z}}_{k,t}^{(2)},\cdots,{\boldsymbol{z}}_{k,t}^{(M)})$ denote the RB allocation vector for device $k$ in round $t$, where ${\boldsymbol{z}}_{k,t}^{(m)} \in \{0,1\}$, ${\boldsymbol{z}}_{k,t}^{(m)}=1$ indicates that the $m$-th resource block is allocated to device $k$, and ${\boldsymbol{z}}_{k,t}^{(m)}=0$ otherwise. Thus, the achievable transmit rate of device $k$ in round $t$ is 
    	\begin{align}
    		R_{n,k}^{U L}&= \sum_{m=1}^{M} \boldsymbol{z}_{n,k}^{(m)} \mathbf{B} \log_2 (1+\frac{\boldsymbol{p}_{n,k} h_o d_{n,k}^{-\gamma}}{N_o}),\nonumber\\
    		&= a_{n,k} \mathbf{B} \log_2 (1+\frac{\boldsymbol{p}_{n,k} h_o d_{n,k}^{-\gamma}}{N_o}),
    		\label{RUL}
    	\end{align}
    \noindent where $\mathbf{B}$, $\boldsymbol{p}_{n,k}$, $h_o$, $N_o$ represent the subcarrier bandwidth, device $n$'s transmission power, channel gain, and thermal noise power, respectively. To be simplify, we denote that $a_{n,k} = \sum_{m=1}^{M} \boldsymbol{z}_{n,k}^{(m)}$. During the local computations, devices transmit the intermediate data to edge server for further processing. After the local computation, devices transmit the updated device-side model to the server. The uploading parameters contain activation $\mathcal{A}$ and updated device-side model $\boldsymbol{\omega}_{c}$. In this paper, due to the activation is much larger than device-side model, so we omit the device-side model uploading period. Before uplink transmission, the activation of device $n$ at round $k$ is quantized with $q_{n,k}$ bits. Let $s_0$ denote the original model size with $q^{max}$ bits. Since $\mu$ is much smaller than $q_{n,k}|\boldsymbol{d}|$ from (\ref{num_bits}), we can approximately calculate the size of the quantized data with $q_{n,k}$ bits as \cite{10542529,10171192}
    	\begin{equation}
    		s_{n,k} = \frac{s_o q_{n,k} }{q^{max}}.
    	\end{equation}
    	Then, the uplink communication time $T_n^{com}$ can be expressed as 
    	\begin{equation}
    		T_n^{com} = \frac{s_{n,k}}{R_n^{UL}}.
    	\end{equation}
    	Thus, the energy consumption is denoted as 
    	\begin{equation}
    		E_n^{com} = P_n^{UL} \cdot T_n^{com}.
    	\end{equation}
    	\subsubsection{Total Latency and Energy}
    	In SFLAM, the overall time and energy cost of device $n$ is
    	\begin{align}
    		T_{n,k} &= T_n^{cmp}+T_n^{com}+T_s^{cmp},\\
    		E_{n,k} &= E_n^{cmp}+E_n^{com}.
    	\end{align}
    	\par The overall system energy cost of all participant devices in $k$-round is defined as follows:
    	\begin{equation}
    		Cost_k = \sum_{n\in\mathcal{V}} a_{n,k} E_{n,k}.\label{energycost}
    	\end{equation}
 
        \section{Convergence Analysis of SFLAM}     \label{section4}
       In this section, we present the convergence analysis of SFLAM, beginning with the assumptions and proposition. Finally, we obtain the convergence result based on activation quantization.  
      \subsection{Assumptions} 
         We start with some assumptions for SFLAM convergence analysis  \cite{10038639,10171192,han2024convergence,li2019convergence,9806308}.
        	\begin{myAssu}\label{assump1}
                   Each device $n$'s loss function $L_n$ is $S$-smooth. Specifically, for all $\boldsymbol{\omega}, \boldsymbol{v} \in \mathbb{R}^w$, the following inequality holds: 
                   $L_n(\boldsymbol{\omega}) \leq L_n(\boldsymbol{v}) + \langle \nabla L_n(\boldsymbol{v}), \boldsymbol{\omega} - \boldsymbol{v} \rangle + \frac{S}{2} | \boldsymbol{\omega} - \boldsymbol{v} |_2^2$.
                   This smoothness assumption is commonly satisfied by various loss functions, including logistic regression, softmax classifiers, and $l_2$-norm regularized linear regression.
        	\end{myAssu}
        	\begin{myAssu}\label{assump2}
                    Each device's loss function, $L_n$, is assumed to be $\mu$-strongly convex for some $\mu \geq 0$. Specifically, we have: $L_n(\boldsymbol{\omega}) \geq L_n(\boldsymbol{v}) + \langle \nabla L_n(\boldsymbol{\omega}), \boldsymbol{\omega} - \boldsymbol{v} \rangle + \frac{\mu}{2} | \boldsymbol{\omega} - \boldsymbol{v} |_2^2.$ 
                \end{myAssu}
        	\begin{myAssu}\label{assump3} 
                    The stochastic gradients are unbiased with the variance bounds: $\mathbb{E}{\xi_n \sim \mathcal{D}_n} \left[ | \nabla L_n(\boldsymbol{\omega}, \xi_n) - \nabla L_n(\boldsymbol{\omega}) |_2^2 \right] \leq \sigma_n^2$, and $\mathbb{E}_{\xi_n \sim \mathcal{D}n} \left[\nabla L_n(\boldsymbol{\omega}, \xi_n)\right] = \nabla L_n(\boldsymbol{\omega})$.
                \end{myAssu}	
        	\begin{myAssu}\label{assump4} 
                    The expected squared norm of the stochastic gradients is bounded by $G^2$, i.e., $\mathbb{E}_{\xi_n \sim \mathcal{D}_n} \left[ | g_n(\boldsymbol{\omega}, \xi_n) |_2^2 \right] \leq G^2$.
        	\end{myAssu}
        	\begin{myAssu}\label{assump5}
                    There exists a constant $\epsilon^2$ such that the divergence between local and global gradients is bounded as follows: $| \nabla L_n(\boldsymbol{\omega}) - \nabla \mathbf{L}(\boldsymbol{\omega}) |_2^2 \leq \epsilon^2$. A larger value of $\epsilon^2$ indicates a higher degree of heterogeneity in the data.
        	\end{myAssu}
        
            \subsection{Proposition}
        	When data is IID across multiple devices, the convergence analysis in \cite{koloskova2020unified} can be directly applied to SFL. However, in the case of heterogeneous data, this theory cannot be applied directly due to the device drift problem. The situation is further complicated by partial device participation, which introduces bias into the training process. While existing FL theories have addressed both data heterogeneity \cite{li2019convergence} and partial participation \cite{wang2022unified}, the convergence analysis of SFL presents unique challenges. This is primarily due to the dual-paced model aggregation and updates occurring at both the device and server sides. Specifically, the performance bound analysis of SFL is more intricate than its conventional FL counterparts, owing to the dual-paced nature of the model aggregation and updates. To tackle this challenge, we decompose the convergence analysis into the server-side and device-side updates. The decomposition is presented below.
        	
        	\begin{myPro}(Convergence Decomposition)
        		Let $\boldsymbol{\omega}^* \triangleq [\boldsymbol{\omega}^{\mathcal{C}}; \boldsymbol{\omega}^{\mathcal{S}}]$ represent the optimal global model that minimizes $\mathbf{L}(\cdot)$, and let $\boldsymbol{\omega}_k \triangleq [\boldsymbol{\omega}_k^{\mathcal{C}}; \boldsymbol{\omega}_k^{\mathcal{S}}]$ denote the global model after $k$ rounds of SFL training. According to \cite{han2024convergence}, we have:
                    \begin{equation} 
                        \mathbb{E}[\mathbf{L}(\boldsymbol{\omega}_k)] - \mathbf{L}(\boldsymbol{\omega}^*) 
                        \leq \frac{S}{2} \left( \mathbb{E}[|\boldsymbol{\omega}_k^{\mathcal{S}} - \boldsymbol{\omega}^{\mathcal{S}}|^2] + \mathbb{E}[|\boldsymbol{\omega}_k^{\mathcal{C}} - \boldsymbol{\omega}^{\mathcal{C}*}|^2] \right). 
                    \end{equation} 
                \end{myPro}This result demonstrates that despite the complexity introduced by dual-paced updates, bounding the SFL performance gap can be achieved by separately bounding the gaps at the server-side and device-side models. Importantly, this decomposition approach is also applicable to other distributed frameworks, such as SL.
            \subsection{Convergence Analysis Based on Activation Quantization} To simplify the derivation of the upper bound on the quantized activation, we introduce the following assumptions and corresponding lemmas. These assumptions provide the theoretical foundation for the convergence analysis of quantized activation and help in deriving the relevant upper bounds.
        	\begin{myAssu}
                    To simplify the subsequent derivations, we assume that the loss function $ L_n(\boldsymbol{\omega}, \mathcal{A}) $ of each device $ n $ is $ \mathcal{L} $-smooth, i.e., for all $ \mathcal{A} \in \mathbb{R}^{\boldsymbol{d}} $, the second-order gradient satisfies: $\|\nabla^2_{\mathcal{A}} L_n(\boldsymbol{\omega}, \mathcal{A})\| \leq \mathcal{L}$. Here, $ \nabla^2_{\mathcal{A}} L_n(\boldsymbol{\omega}, \mathcal{A}) $ represents the Hessian of $ L_n $ with respect to $ \mathcal{A} $.
        		\label{assumption6}
        	\end{myAssu}
        	
        	\begin{myLem}\label{Lemma-2} For quantized activation $ \mathcal{Q}(\mathcal{A}_n(\xi_n)) $, the variance of the stochastic gradients satisfies:
        		\begin{align}
        			&\mathbb{E}_{\xi_n \sim \mathcal{D}_n} \mathbb{E}_Q \left[ \| \nabla L_n(\boldsymbol{\omega}, \mathcal{Q}(\mathcal{A}_n(\xi_n))) - \nabla L_n(\boldsymbol{\omega}) \|_2^2 \right] \notag \\
        			\leq & \sigma_n^2 + \mathcal{L}^2 \cdot \delta \| \mathcal{A}_n(\xi_n) \|^2.\nonumber
        		\end{align}
        	\end{myLem}
        	\begin{proof}
        		The detailed proof is provided in the Appendix \ref{proof-lemma2}.
        	\end{proof}
        	
        	\begin{myLem}\label{Lemma-3}For quantized activation $ \mathcal{Q}(\mathcal{A}_n(\xi_n)) $, the norm of the stochastic gradients satisfies:
        		\begin{align}
        			\mathbb{E}_{\xi_n \sim \mathcal{D}_n} \mathbb{E}_Q \left[ \| g_n(\boldsymbol{\omega}, \mathcal{Q}(\mathcal{A}_n(\xi_n))) \|_2^2 \right] \leq G^2 + \mathcal{L}^2 \cdot \delta \| \mathcal{A}_n(\xi_n) \|^2.\nonumber
        		\end{align}
        	\end{myLem}
        	\begin{proof}
        		The detailed proof is provided in the Appendix \ref{proof-lemma3}.
        	\end{proof}
        	\begin{myTheo}
        		Let Assumptions 1 to 6 hold \cite{han2024convergence}, then the following convergence bound holds:
        		\begin{align}
        			&\mathbb{E}[\mathbf{L}(\boldsymbol{\omega}_k)] - \mathbf{L}(\boldsymbol{\omega}^*) \nonumber \\
        			&\leq \frac{8SN}{\mu^2(\gamma + k)} \left( \sum_{n=1}^N \rho_n^2 \left[ 2\sigma_n^2 + G^2 + \frac{G^2}{a_n} \right] \right) \nonumber \\
        			&\quad + \frac{768S^2}{\mu^3(\gamma + k)(\gamma + 1)} \left( \sum_{n=1}^N \rho_n \left[ 2\sigma_n^2 + G^2 \right] \right) \nonumber \\
        			&\quad + \frac{S(\gamma + 1)}{2(\gamma + K)} \left( \mathbb{E}[\|\boldsymbol{\omega}_0 - \boldsymbol{\omega}^{*}\|^2] \right) \nonumber\\
                        &\quad+ \mathcal{L}^2 \cdot \delta A^2 B. \nonumber
        		\end{align}
        		where $\delta = \frac{1+\sqrt{2d-1}}{2}\frac{1}{2^q-1}$, $A = \|\mathcal{A}_n(\xi_n)\|^2$, and $ B = \frac{8SN}{\mu^2(\gamma + k)} \sum_{n=1}^N \rho_n^2 \left( 3 + \frac{1}{a_n} \right) + \frac{768S^2}{\mu^3(\gamma + k)(\gamma + 1)} \sum_{n=1}^N 3a_n$.
        		\label{theorem1}
        	\end{myTheo}
        	\begin{proof}
        		The detailed proof is provided in the Appendix \ref{proof-theorem1}.
        	\end{proof}
                \begin{figure*}[ht]
        		\centering
                    \subfloat[CIFAR-10 Dir(0.3) lr=0.001]{
        			\includegraphics[width=0.23\textwidth]{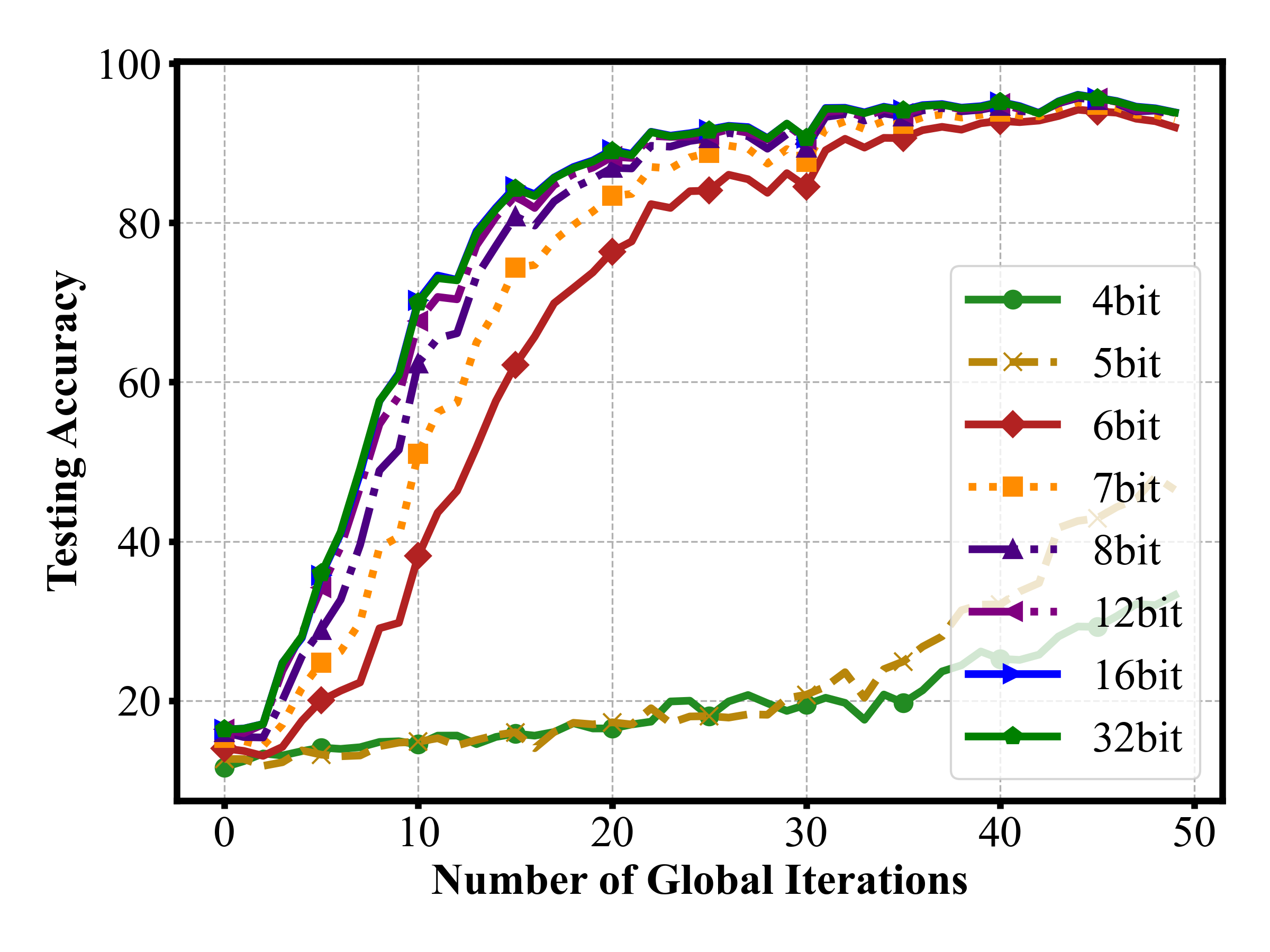}\label{fig:subfig1}
        		}
        		\subfloat[CIFAR-10 Dir(0.5) lr=0.001]{
        			\includegraphics[width=0.23\textwidth]{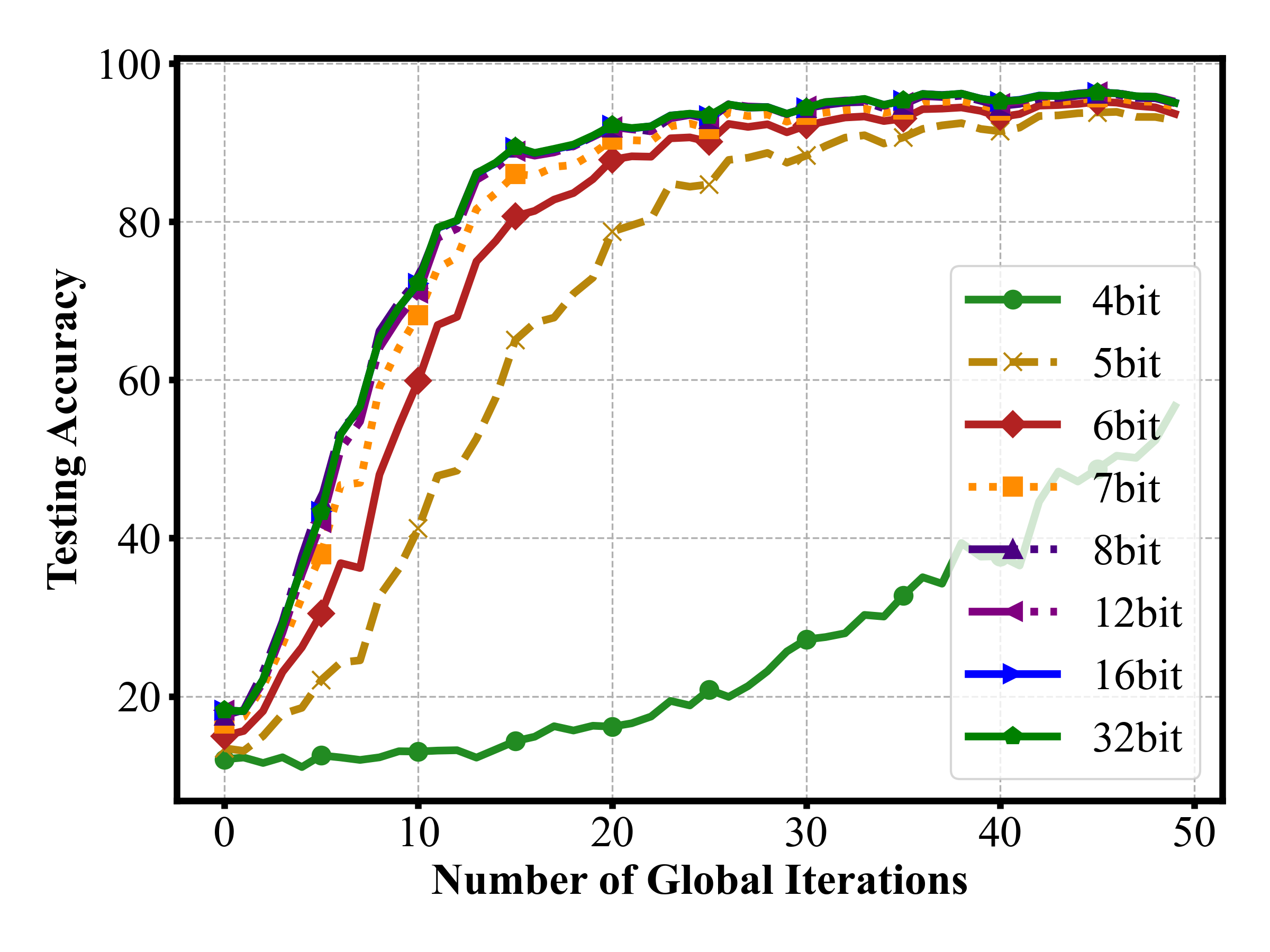}\label{fig:subfig2}
        		}
                    \subfloat[CIFAR-10 Dir(0.3) lr=0.01]{
        			\includegraphics[width=0.23\textwidth]{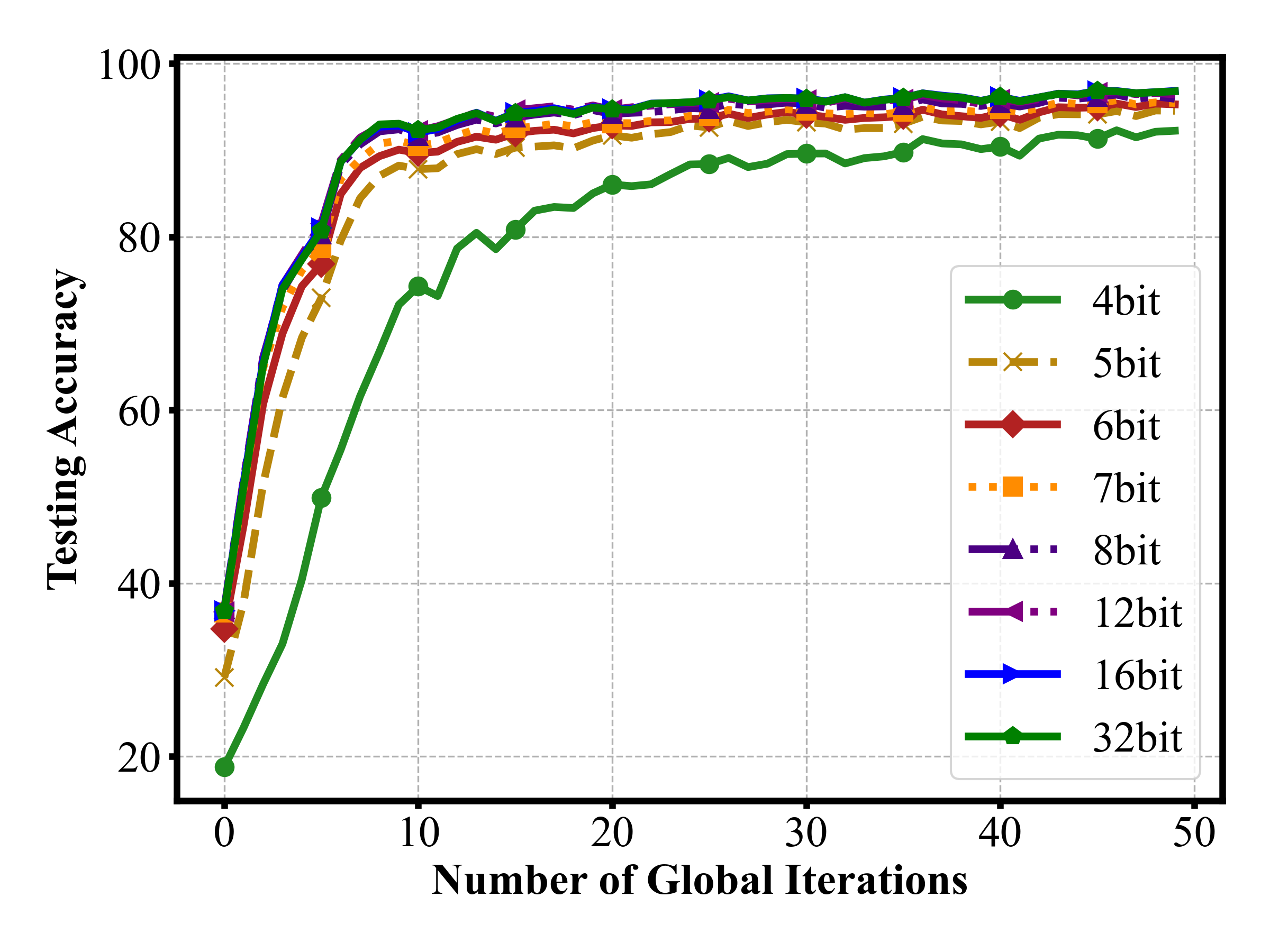}\label{fig:subfig3}
        		}
        		\subfloat[CIFAR-10 Dir(0.5) lr=0.01]{
        			\includegraphics[width=0.23\textwidth]{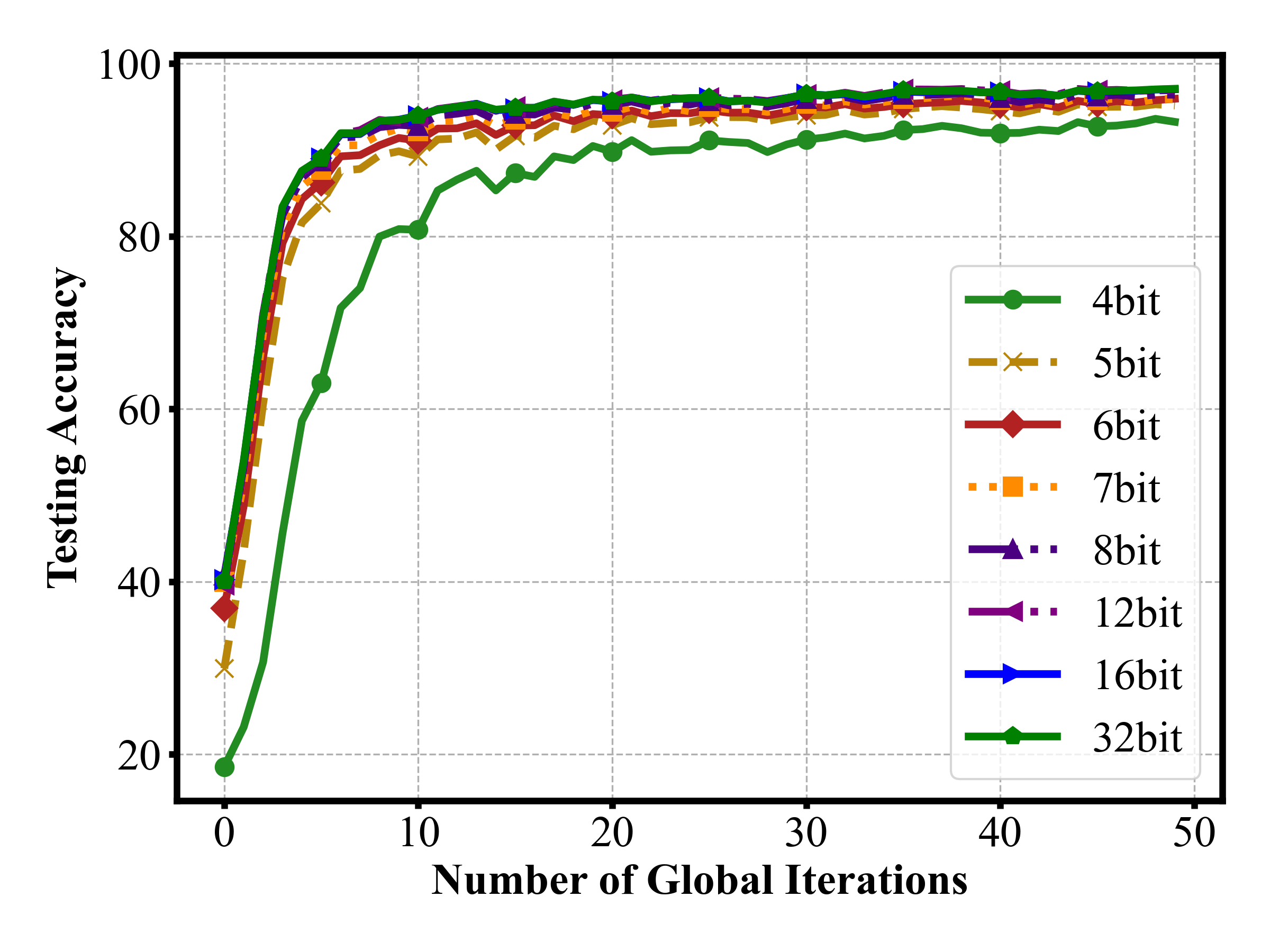}\label{fig:subfig4}
        		}
                    \hfill
        		\subfloat[CIFAR-100 Dir(0.3) lr=0.001]{
        			\includegraphics[width=0.23\textwidth]{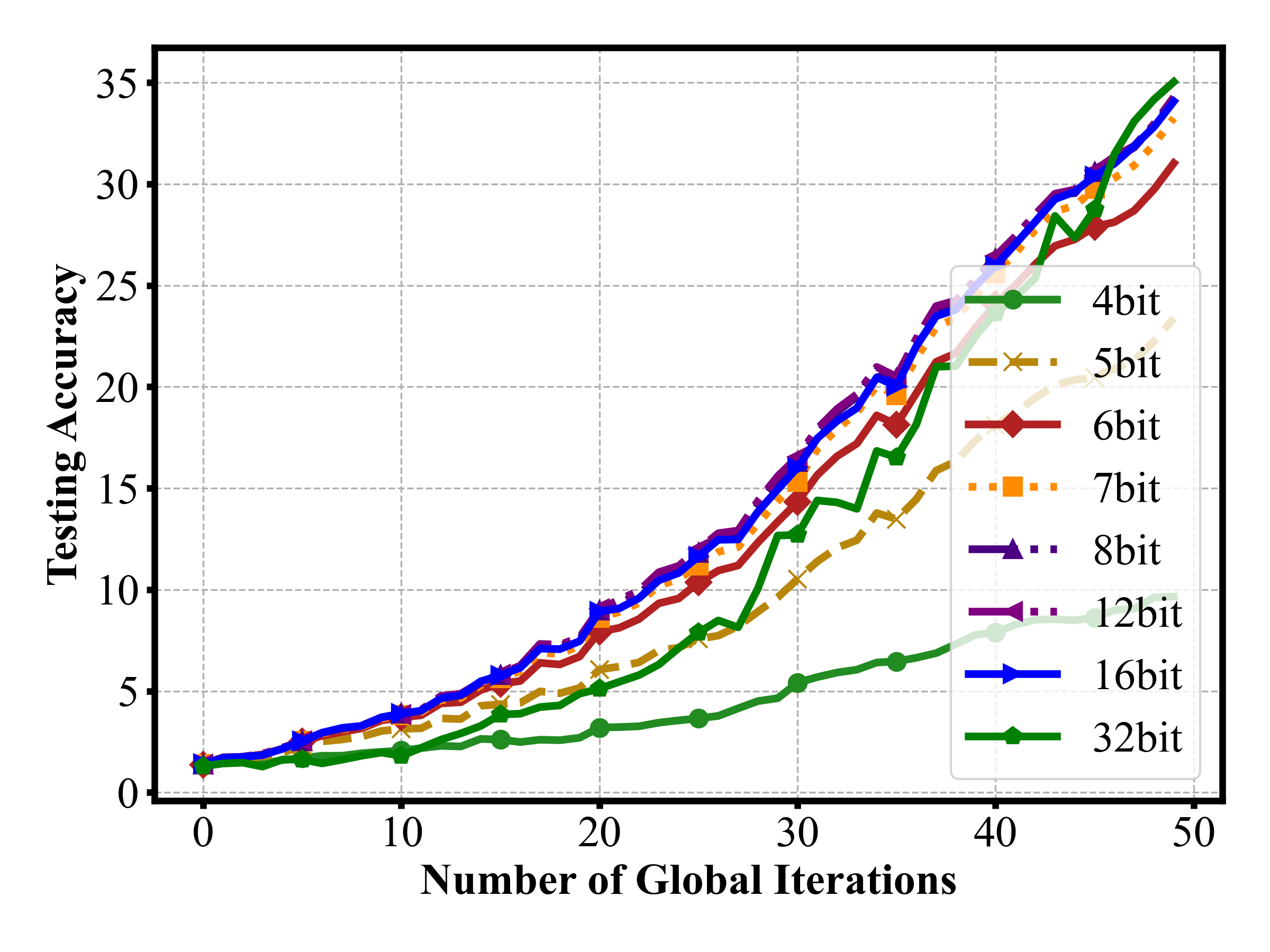}\label{fig:subfig5}
        		}
        		\subfloat[CIFAR-100 Dir(0.5) lr=0.001]{
        			\includegraphics[width=0.23\textwidth]{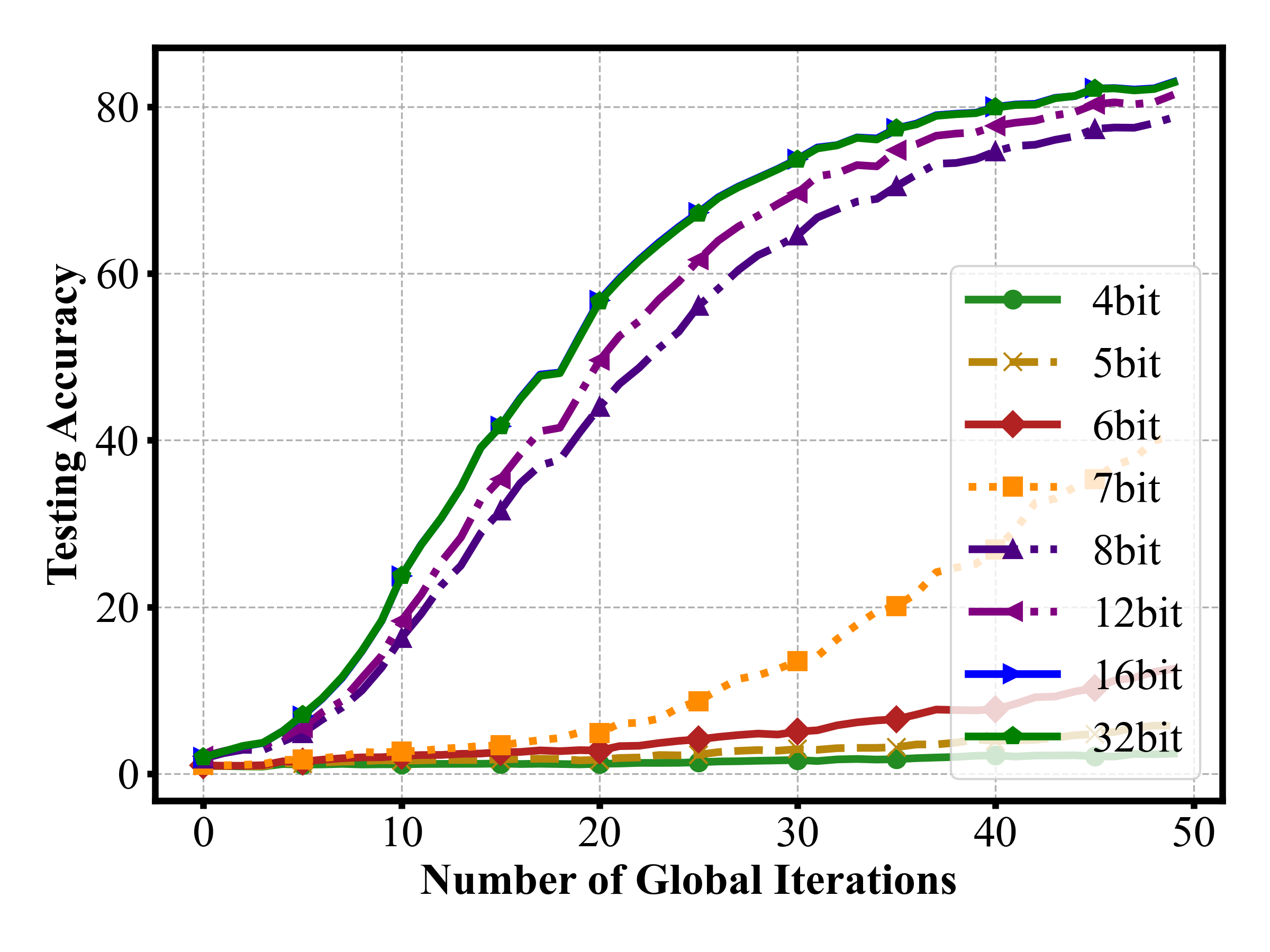}\label{fig:subfig6}
        		}
                    \subfloat[CIFAR-100 Dir(0.3) lr=0.01]{
        			\includegraphics[width=0.23\textwidth]{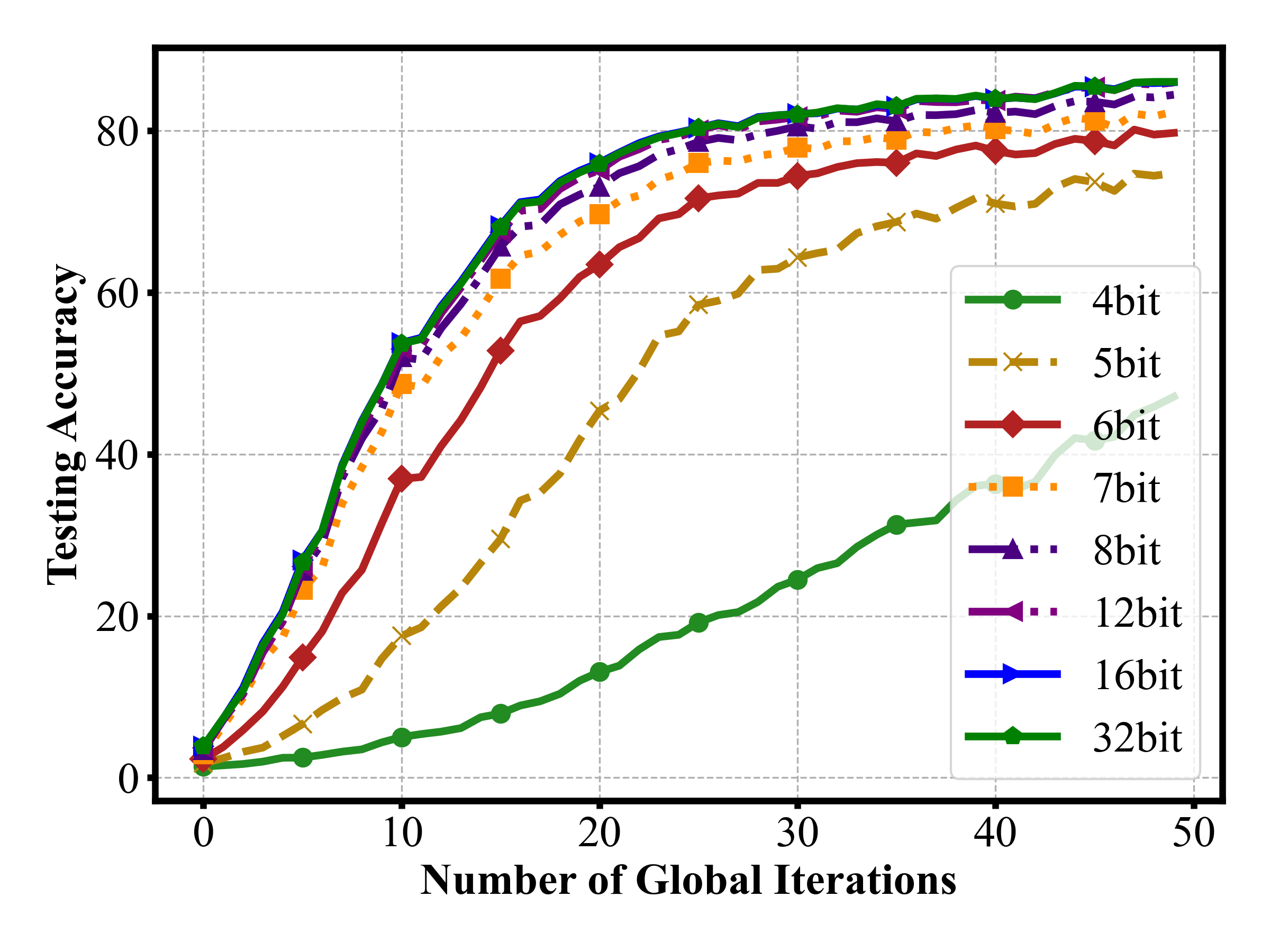}\label{fig:subfig7}
        		}
        		\subfloat[CIFAR-100 Dir(0.5) lr=0.01]{
        			\includegraphics[width=0.23\textwidth]{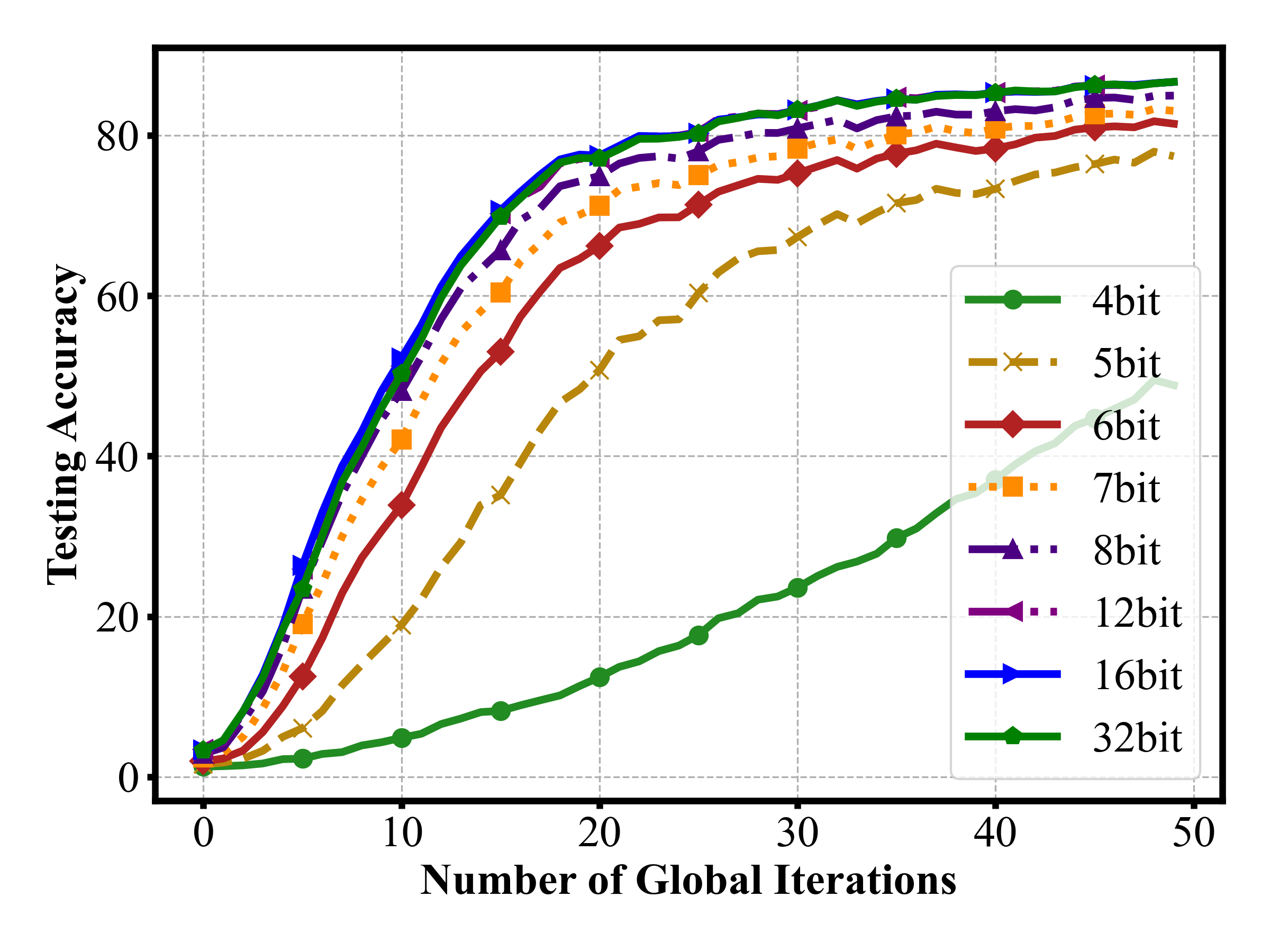}\label{fig:subfig8}
        		}
        		\caption{Testing accuracy with random selection of 10 out of 50 devices under different Dirichlet distributions. $Dir(\alpha)$ and quantization settings}
        		\label{fig:main}
        	\end{figure*}   
        	We observe that as the quantization bit $ q $ increases, the upper bound of $ \mathbb{E}[\mathbf{L}(\boldsymbol{\omega}_k)] - \mathbf{L}(\boldsymbol{\omega}^*) $ decreases, resulting in improved convergence of SFL. This observation aligns with the understanding that a smaller quantization bit may lead to more missing information, thus potentially reducing the model's accuracy \cite{10542529,9611373}. This theoretical result is consistent with our empirical observation in Fig. \ref{fig:main} of the quantification of activation. As we can see, when the quantization bit is less than 4, the accuracy and convergence speed is too slow to train a model, so we set the quantization bit $q^{min}=5$.

        \section{Problem Formulation}
        \label{section5}
            Based on the preceding observations, we deduce that augmenting quantization bits yields enhanced model accuracy at the expense of increased communication overhead. Although model accuracy improves with the number of quantization bits, the marginal gains decrease. Consequently, an inherent trade-off exists between model performance and system cost. To capture this relationship, we define model performance $P_k$ in the $k-th$ round as follows:
        	\begin{equation}
        		P_{n,k} = \log(1+q_{n,k})
        	\end{equation}
	    \noindent where $\rho_n = \frac{D_n}{\sum_{n \in \mathcal{N}_k}D_n}$ represents the proportion of local model $\boldsymbol{\omega}_n$. To balance the model performance and system cost, we define system efficiency as follows:
	\begin{equation}
		\mathcal{EFF}_k = \sum_{\forall n \in \mathcal{V}}a_{n,k} \frac{P_{n,k}}{Cost_{n,k}},
        \label{eq-objective-function}
	\end{equation}
	 \noindent where $a_{n,k}=\sum_{m=1}^{M} \boldsymbol{z}_{n,k}$ indicates whether device $n$ participates in the $k$-th round. There is a tradeoff between model performance and system cost. As quantization increases, the model performance $P_{n,k}(q_{n,k})$ improves; however, this comes at the expense of a higher system cost $Cost_{n,k}(q_{n,k})$. To minimize system cost while maintaining satisfactory model performance, we define the objective function as follows:
	\begin{subequations}
		\label{P}
		\begin{align}
			\mathcal{P}:\quad &\underset{\mathbf{z}, \mathbf{p}, \mathbf{q}}{\max} \quad \mathcal{EFF}_k\nonumber \\
			\text{s.t.} \quad &T_{n,k} \leq T_{\text{max}}, \quad \forall n \in \mathcal{V}, \forall k, \label{P1a}\\
			&\sum_{n=1}^{|\mathcal{V}|} z_{n,k}^{(m)}  \leq M, \quad \forall m \in \mathcal{M}, \forall k,\label{P1b}\\
			&\sum_{m=1}^{M} z_{n,k}^{(m)} \leq 1, \quad \forall n \in \mathcal{V}, \forall k, \label{P1c}\\
			&z_{n,k}^{(m)}  \in \{0,1\}, \quad \forall n \in \mathcal{V}, \forall k, \label{P1d}\\
			&p^{\text{min}} \leq p_{n,k} \leq p^{\text{max}}, \quad \forall n \in \mathcal{V}, \forall k, \label{P1e}\\
			&q^{\text{min}} \leq q_{n,k} \leq q^{\text{max}}, \quad \forall n \in \mathcal{V}, \forall k, \label{P1f}\\
			&a_{n,k} = \sum_{m=1}^{M} \boldsymbol{z}_{n,k}^{(m)} \in \{0,1\}, \quad \forall n \in \mathcal{V}, \forall k, \label{P1g}
		\end{align}
	\end{subequations}
	Due to the min-max nature of $\mathcal{EFF}_k$ and the integer constraints on $\boldsymbol{p}, \mathbf{\boldsymbol{z}}$ and $\mathbf{q}$, this problem is non-convex and NP-hard. To solve it, we propose the following approaches.
	
	\section{Optimal Power Control, Quantization Management and Bandwidth Allocation}
            \label{section6}
            To solve the problem $\mathcal{P}$, we first decompose it into three sub-problems: power control, quantization management and resource allocation.
    	\subsection{Optimal Power Control}
    	For any given RB allocation policy $\mathbf{\boldsymbol{z}}$, it is straightforward to see that the power control policies of devices do not affect each other and independently contribute to the objective function. Therefore, the power control policy for each device can be solely optimized by itself.
    	\begin{subequations}
    		\label{SUBP1}
    		\begin{align}
    			\mathcal{SUBP}1:\quad &\underset{\mathbf{p}}{\min} \quad Cost_k\nonumber\\
    			\text{s.t.} \quad & (\ref{P1a}), (\ref{P1e}).\nonumber
    		\end{align}
    	\end{subequations}
    	Based on the greedy approach, minimizing the total energy consumption across all devices is equivalent to minimizing the energy consumption of each device individually. Since the energy function is separable, optimizing the energy for each device independently results in the global minimum for the system. For every device $n$, we have:
    	\begin{subequations}
    		\label{SUBP1-1}
    		\begin{align}
    			\mathcal{SUBP}1':\quad  &\underset{\mathbf{p}}{\min} \quad  E_{n}^{cmp}+\frac{p_n s_{n}}{B \log_2(1+\frac{p_n h_o d_{n}^{-\gamma}}{N_o})}, \nonumber\\
    			\text{s.t.} \quad & T_{n}^{cmp}+ T_{s}^{cmp} + \frac{s_{n}}{B\log_2(1+\frac{p_n h_o d_{n,k}^{-\gamma}}{N_o})} \leq T_{\text{max}}, \nonumber\\
    			&p_{min} \leq p_{n}\leq p_{max}.  \nonumber
    		\end{align}
    	\end{subequations}
        \par It can be observed that the second term in the objective function and the constraint involving $\log_2(1 + \frac{p_n h_o d_n^{-\gamma}}{N_o})$ are non-convex in terms of the power variable $p_n$, which makes $\mathcal{SUBP}1'$ a non-convex optimization problem. We define $p_n^{i}$ as the power at the $i$-th iteration. The non-convex part of the objective function is:
    	\begin{equation}
    		E_{n}^{com} = \frac{p_n s_n}{B \log_2\left(1 + \frac{p_n h_o d_n^{-\gamma}}{N_o}\right)}.
    	\end{equation}
        \par To obtain the approximate upper bound, $E_{n}^{com}$ can be approximated by its first-order Taylor expansion $\hat{E}_{n}^{com}$ around $p_n^{i}$, which is given by:
    	\begin{equation}
    		\hat{E}_{n}^{com}(p_n^i,p_n) = E_{n}^{com}(p_n^{i}) +\frac{dE_n^{com}}{dp_n^i}\cdot (p_n - p_n^{i}),
    		\label{hat_e}
    	\end{equation}
    	where $\frac{dE_n^{com}}{dp_n^i}$ denotes the first-order derivative of $ E_{n}^{com}$ evaluated at $p_n^{i}$: $\frac{dE_n^{com}}{dp_n^i} =	\frac{A_0}{\log_2(1+B_0 p_n^i)}-\frac{A_0 B_0 p_n^i}{\ln2 (1+B_0 p_n^i) (\log_2(1+B_0 p_n^i))^2}$.\par
    	By substituting the non-convex part $E_n^{com}$ of the objective function with its first-order approximation $\frac{dE_n^{com}}{dp_n}$, we convert the objective function into a convex form at each iteration. Thus, the original non-convex problem becomes a convex problem at each SCA iteration.\par
    	Similarly, for the constraints involving the logarithmic terms, we define the non-convex part of the constraint as $T_n^{com} = s_n/\left(B \log_2\left(1 + \frac{p_n h_o d_n^{-\gamma}}{N_o}\right)\right)$. To approximate the upper bound of $T_n^{com}$, we apply the first-order Taylor expansion $\hat{T}_n^{com}(p_n^{i}, p_n)$ around $p_n^{i}$, which is given by 
    	\begin{equation}
    		\hat{T}_n^{com}(p_n^{i}, p_n) = T_n^{com}(p_n^{i}) +\frac{dT_n^{com}}{dp_n^i} \cdot ( p_n - p_n^{i}),
    		\label{hat_t}
    	\end{equation}
        where first-order derivative of $T_n^{com}$ evaluated at $p_n^i$ is $\frac{dT_n^{com}}{dp_n^i} = \frac{-A_0 B_0 ln(2)}{(1+B_0 p_n^i)(ln(1+B_0p_n^i))^2}$, where $A_0 =s_n/B$ and $B_0 = \frac{h_od_n^{-\gamma}}{N_o}$. This ensures that the constraint becomes convex by replacing the non-convex part with the first-order approximation, transforming the constraint into $\hat{T}_n^{com}(p_n^{i}, p_n) \leq T_{\text{max}}$.
    	Thus, the optimization problem \(\mathcal{SUBP}1'\) becomes convex at each SCA iteration, and we can iteratively solve that.   	
    	\begin{algorithm}[t]
    		\caption{Optimal Power Control using SCA Method} 
    		\label{alg1}
    		\begin{algorithmic}
    			\REQUIRE Set the initial value of $\mathbf{z}, \mathbf{q}$, max time constraints $T_{max}$, the initial uplink power $p_{n,k}^0$ of vehicle $n$, iteration round $i=0$, the accuracy requirement $\epsilon$.
    			\REPEAT
    			\STATE calculate $\hat{E}_{n}^{com}(p_n^i,p_n)$ and $\hat{T}_{n}^{com}(p_n^i,p_n)$ according to Eqn. (\ref{hat_e}) and (\ref{hat_t}) respectively;
    			\STATE solve $\mathcal{SUBP}1$ by substituting $E_n^{com}(p_n)$ with  $\hat{E}_{n}^{com}(p_n^i,p_n)$, substituting $T_n^{com}(p_n)$ with $\hat{T}_{n}^{com}(p_n^i,p_n)$, and achieve the optimal solution $p_n^{i,*}$
    			\STATE $p_n \rightarrow p_n^{i,*}$, $i \rightarrow i+1$
    			\UNTIL{$\left\|p_n^i-p_n^{i-1}\right\| \leq \epsilon$}
    			\ENSURE	Optimal transmission power $\phi_n^*$.
    		\end{algorithmic}
    	\end{algorithm}    
    
    \subsection{Optimal Quantization Management}
    	We consider the optimization problem for optimal quantization, denoted as $\mathcal{SUBP}2$, formulated as follows:
    	\begin{subequations}
    		\label{SUBP2}
    		\begin{align}
    			\mathcal{SUBP}2:\quad  &\underset{\mathbf{q}}{\max} \quad \mathcal{EFF}_k = \sum_{n \in \mathcal{V}}a_{n,k}\frac{\log(1+q_{n,k})}{E_n^{cmp}+E_n^{com}}\nonumber \\
    			\text{s.t.} \quad & (\ref{P1a}), (\ref{P1f}). \nonumber
    		\end{align}
    	\end{subequations}
    	Using the greedy algorithm, the objective is to maximize the total system efficiency, which is equivalent to maximizing the efficiency of each device individually. Since the efficiency function is separable, optimizing each device’s efficiency independently ensures the global maximum for the system. This approach works because improving individual device efficiency directly enhances overall system performance. This translates into minimizing the energy for each device individually. For every device $n \in \mathcal{N}_k$, the problem becomes:
    
        \begin{align}
            \label{SUBP2-2}
            \mathcal{SUBP}2':\quad  &\underset{\mathbf{q}}{\max} \quad  \mathcal{EFF}_{n,k} = \frac{\log{(1+q_{n,k})}}{E_n^{cmp}+\frac{p_n s_o}{q_{max}R_{n}^{UL}}q_{n,k}}\nonumber \\
            \text{s.t.} \quad & T_n^{cmp}+\frac{s_o}{R_{n}^{UL}q_{max}}q_{n,k} \leq T_{max},\nonumber\\
            &q^{\text{min}} \leq q_{n,k} \leq q^{\text{max}}.\nonumber
        \end{align}
    
         \begin{algorithm}[t]
            \caption{Optimal Quantization Management for Device $n$}
            \label{alg2}
            \begin{algorithmic}
                \REQUIRE Set the maximum time constraint $T_{max}$, iteration round $i = 0$.
                \FOR{each $q_{n,k} \in [q^{\min}, q^{\max}]$}
                    \STATE Compute $\mathcal{EFF}_{n,k}$ and update optimal $q_{n,k}^{*}$.
                \ENDFOR
                \ENSURE Optimal quantization level $q_{n,k}^{*}$.
            \end{algorithmic}
        \end{algorithm}
    
    	To solve the subproblem $\mathcal{SUBP}2'$ in a computationally efficient manner, we employ an exhaustive search approach, which iteratively evaluates all possible quantization levels within the defined bounds of $q^{\min}$ and $q^{\max}$ for each device. This method allows us to find the optimal quantization level $q_{n,k}$ for each device individually. Specifically, for each device $n \in \mathcal{N}_k$, we evaluate the objective function across a discretized set of possible values for $q{n,k}$, ensuring that we respect the constraints on computation time and quantization limits. This procedure is detailed in Alg. \ref{alg2}. By optimizing the efficiency of each device independently, we maximize the overall system efficiency, as the total efficiency is the sum of the individual efficiencies.
        
    \subsection{Optimal Bandwidth Allocation}
	The objective is to find a matching of devices and sub-channels that minimizes the total energy consumption while respecting the constraints on the resource block allocations. Specifically, the allocation of resource blocks to each device $n$ is represented by the vector $\boldsymbol{z}_{n,k}$. The optimization problem can be formulated as:
	\begin{subequations}
		\label{SUBP3}
		\begin{align}
			\mathcal{SUBP}3:\quad  &\underset{\mathbf{z}_k}{\max} \quad \sum_{n=1}^N\sum_{m=1}^{M}z_{n,k}^{(m)}\frac{\log(1+q_n,k)}{E_{n,k}}\nonumber \\
			\text{s.t.} \quad & (\ref{P1a}), (\ref{P1b}), (\ref{P1c}), (\ref{P1d}).\nonumber
		\end{align}
	\end{subequations}
	Problem $\mathcal{SUBP}$3 is a typical non-linear integer programming problem which is difficult to solve. Below we transform it into a minimal weight perfect bipartite matching problem and find its optimal solution within polynomial time. The bipartite matching problem is to find a matching (i.e., a set of edges chosen such that no two edges share an endpoint.) with the maximum weight for the bipartite graph, where the weight is the summation of all the edges in the matching \cite{10444714,schrijver2003combinatorial}.\par 
    \begin{algorithm}[t]
        \caption{Joint Optimal Algorithm using BCD method} 
        \label{alg4}
        \begin{algorithmic}[1]
        \REQUIRE Set $i=0$, $\epsilon_1,\epsilon_2,\epsilon_3>0$. 
        \REPEAT
        \STATE Calculate the optimal uplink power \textbf{$\mathbf{p}^{i}$} from $\mathcal{SUBP}1$ at given $\mathbf{q}^{i-1},\mathbf{z}^{i-1}$ using Algorithm \ref{alg1};
        \STATE Calculate the optimal quantization bits $\mathbf{q}^{i}$ from $\mathcal{SUBP}2$ at given $\mathbf{p}^{i},\mathbf{z}^{i-1}$ by applying Algorithm \ref{alg2};
        \STATE Assign the optimal bandwidth allocation $\mathbf{z}^{i}$ at given $\mathbf{p}^{i},\mathbf{q}^{i}$
        \UNTIL{$\Vert \mathbf{p}^{i}-\mathbf{p}^{i-1}\Vert<\epsilon_1$,  $\Vert \mathbf{q}^{i} -\mathbf{q}^{i-1}\Vert<\epsilon_2$, $\Vert \mathbf{z}^{i}-\mathbf{z}^{i-1}\Vert<\epsilon_3$};
        \STATE $i =  i + 1$;
        \ENSURE $\mathbf{p}^{i},\mathbf{q}^{i},\mathbf{z}^{i}$. 
        \end{algorithmic}
    \end{algorithm}
	To transform $\mathcal{SUBP}$3 into a bipartite matching problem, we construct a complete and balanced bipartite graph $\mathcal{G} = \{\mathcal{V},\mathcal{E}\}$, where $\mathcal{V}=\mathcal{N}\cup\overline{\mathcal{M}}$ is the vertex set. In $\mathcal{G}$, each vertex $n$ in $\mathcal{N}$ corresponds a device $n$. $\overline{\mathcal{M}}=\mathcal{M}\cup\mathcal{M}_{v}$ is an extended set of $\mathcal{M}$, where each vertex $m$ in $\mathcal{M}$ correspondings to RB $m$. $\mathcal{M}_{v}$ is the virtual vertex set used to construct a balanced bipartite graph $\mathcal{G}$, which makes the size of $\overline{\mathcal{M}}$ equal to the size of $\mathcal{N}$, i.e., $|\overline{\mathcal{M}}|=|\mathcal{N}|$. The weight of edged in $\mathcal{G}$ is given by
	\begin{equation}
		\Delta_n =
		\begin{cases} 
			\frac{\log(1+q_{n,k})}{\frac{\kappa_1 f_n^2 \gamma_d}{C_n D_n} + \frac{p_n \cdot s_n}{B \log_2\left(1 + \frac{p_n h_o d_n^{-\gamma}}{N_o}\right)}}, & \text{if } (\ref{P1a}), \ \forall n, m, \\[10pt]
			0, & \text{otherwise.}
		\end{cases}
		\label{eqn:delta_km}
	\end{equation}
	Note that, in this work, we assume that the number of devices exceeds the number of RBs. According to the above-defined bipartite graph $\mathcal{G}$, we transform $\mathcal{SUBP}3$ into minimum weight perfect bipartite matching problem, which aims to find a perfect matching $\mathcal{H}$ of $\mathcal{G}$ minimizating $\sum_{e \in \mathcal{H}}\Delta_{n}$. Let $\theta_{n,m} \in \{0,1\}$ be the edge connecting vetex $n$ $(n\in \mathcal{N})$ and vertex $m$ $(m \in \overline{\mathcal{\mathcal{M}}})$, where $\theta_{n,m}=1$ denote that RB $m$ is allocated to device $n$, and $\theta_{n,m}=0$ otherwise. Thus, we formulate the bipartite matching problem as the following optimization problem.
	\begin{subequations}
		\label{SUBP3-2}
		\begin{align}
			\mathcal{\hat{SUBP}}3:\quad  &\underset{\theta}{\max} \quad \sum_{n=1}^N\sum_{m=1}^{M}\theta_{n,m} \Delta_{n} \nonumber\\
			\text{s.t.} \quad 
			& \sum_{n=1}^{N} \theta_{n,m} = 1, \forall m \in \overline{\mathcal{M}},\label{SUBP3-2a}\\
			& \sum_{m=1}^{\overline{\mathcal{M}}} \theta_{n,m} \leq 1, \forall n \in \mathcal{V},\label{SUBP3-2b}\\
			& \theta_{n,m} = \{0,1\}, \forall n \in \mathcal{V}, \forall m \in \overline{\mathcal{M}}, \label{SUBP3-2c}
		\end{align}
	\end{subequations}
	It is worth mentioning that any solution to problem $\hat{SUBP}3$ corresponds to a perfect matching of graph $\mathcal{G}$. However, problem $\hat{SUBP}3$ is still difficult to solve. By relaxing the integrality constraint (\ref{SUBP3-2c}), we can obtain the following linear programming problem:

	\begin{subequations}
		\label{SUBP3-3}
		\begin{align}
			\mathcal{\overline{SUBP}}3:\quad  &\underset{\mathbf{\theta}}{\max} \quad \sum_{n=1}^N\sum_{m=1}^{M}\theta_{n,m} \Delta_{n} \nonumber\\
			\text{s.t.} \quad &(\ref{SUBP3-2a}), (\ref{SUBP3-2b}) \nonumber\\
			& 0 \leq \theta_{n,m} \leq 1, \forall n \in \mathcal{V}, \forall m \in \overline{\mathcal{M}}.\nonumber
		\end{align}
	\end{subequations}
	
	Problem $\hat{SUBP}$3 is the linear programming relaxation of problem $\hat{SUBP}3$, which can be solved by using the interior point method with time complexity of $\mathcal{O}((NM)^{3.5})$ since it has $NM$ variables\cite{karmarkar1984new}.
     \begin{figure*}[t]
        \centering
        \subfloat[Energy after solving $\mathcal{SUBP}$2.]{
            \includegraphics[width=0.32\textwidth]{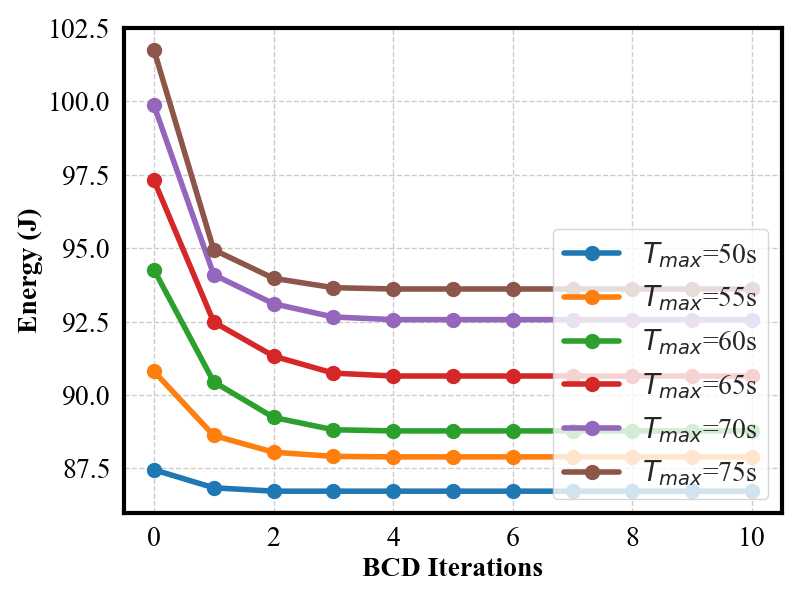}
            \label{fig:resource1}
        }
            \subfloat[Energy after solving $\mathcal{SUBP}$3.]{
            \includegraphics[width=0.30\textwidth]{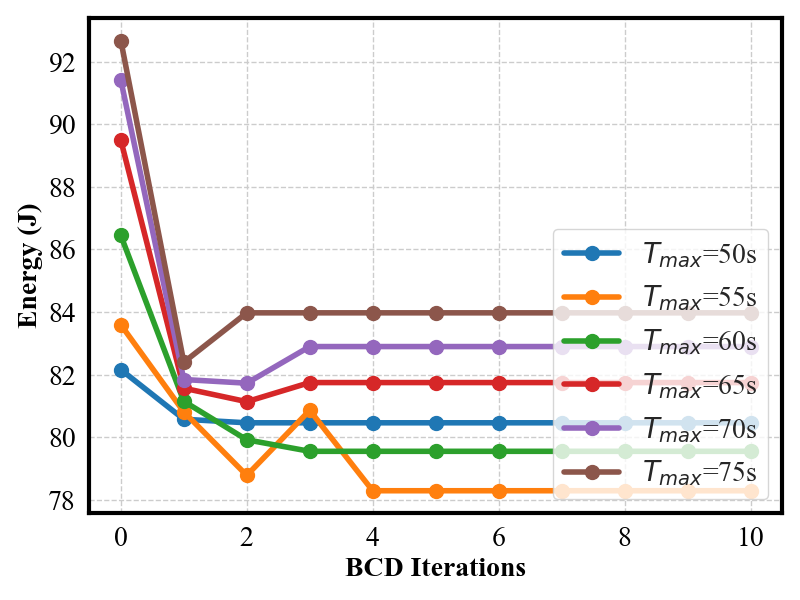}
            \label{fig:resource2}
        }
        \subfloat[Efficiency after solving $\mathcal{SUBP}$3.]{
            \includegraphics[width=0.30\textwidth]{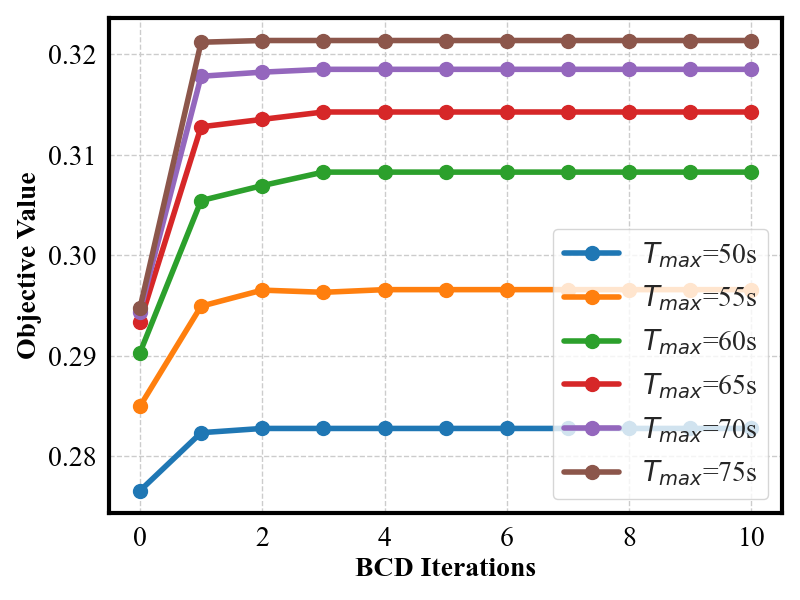}
            \label{fig:resource3}
        }
                \hfill
        \subfloat[Energy versus $T_{max}$.]{
            \includegraphics[width=0.30\textwidth]{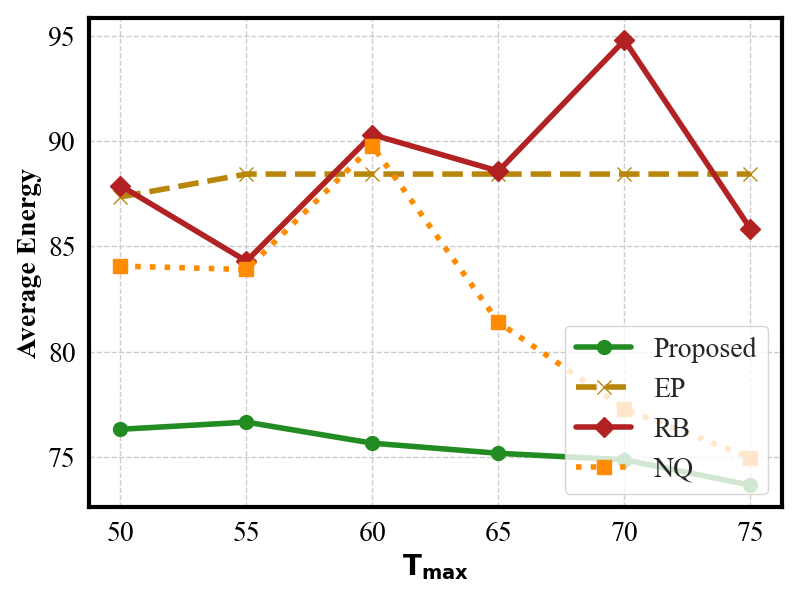}
            \label{fig:resource4}
        }
            \subfloat[Efficiency versus $T_{max}$.]{
            \includegraphics[width=0.30\textwidth]{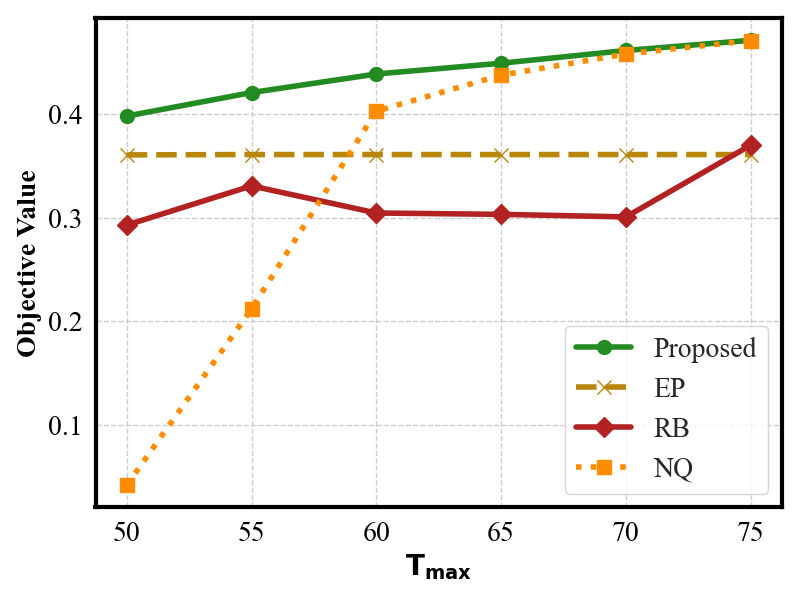}
            \label{fig:resource5}
        }
        \subfloat[Quantization bits versus $T_{max}$.]{
            \includegraphics[width=0.30\textwidth]{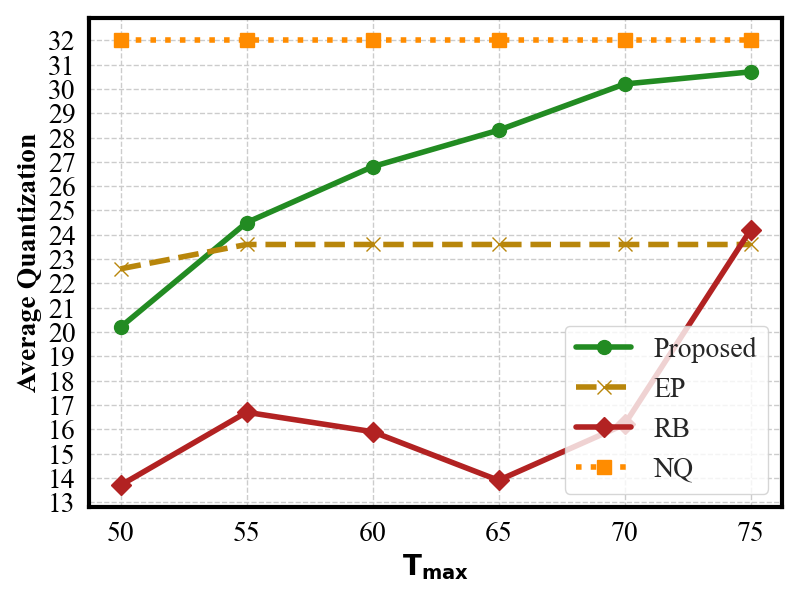}
            \label{fig:resource6}
        }
        \caption{Average energy consumption and objective value after solving three subproblems under different $T_{max}$.}
        \label{fig:resource_allocation1}
    \end{figure*}	  
     \subsection{Joint BCD Algorithm and Complexity Analysis}
    Although there is not a closed-form solution for the optimal power and wireless resource allocation, the block coordinate descent (BCD) approach can be used to find the optimal solutions. In Algorithm \ref{alg4}, $i$ initially define as $i=0$. We solve $\mathcal{SUBP}1$, the optimal uplink transmission power $\mathbf{p}^{i}$ is obtained by fixing $\mathbf{q}^{i}, \mathbf{z}^{i}$ in the $i$-th iteration. The optimal quantization bits $\mathbf{q}^{i}$ is calculated by given $\mathbf{p}^{i},\mathbf{z}^{i-1}$. Then the bandwidth allocation $\mathbf{z}^{i}$ can be calculated based on $\mathbf{p}^{i},\mathbf{q}^{i}$ by solving $\mathcal{SUBP}3$. The loops end until the differences meet the threshold requirement $\epsilon_1, \epsilon_2, \epsilon_3$.\par
    The computation complexity of Algorithm \ref{alg4} mainly composed with the three sub-problems. According to \cite{boyd2004convex}, the $\mathcal{SUBP}1$ use the SCA method, and the computational complexity is $\mathcal{O}(I_{SCA}N^3))$, where $N$ is the number of variables. The complexity of $\mathcal{SUBP}2$ using greedy algorithm is $\mathcal{O}(N)$. The $\mathcal{SUBP}3$ is solved using the interior point method, and its complexity is $\mathcal{O}((NM)^{3.5})$ due to having $NM$ variables \cite{karmarkar1984new}. Therefore, the overall computation complexity of the overall algorithm is $\mathcal{O}(I_{BCD}*(I_{SCA}N^3+N+(NM)^{3.5})$, where $I_{BCD}$ is the number of iterations of BCD algorithm.
    
\section{Simulation Results}\label{section7}
    \subsection{Simulation Setup}
        \subsubsection{Experimental Setting for Fine-Tuning and Datasets}
        We utilize the ViT-Base/32 model for the distributed image classification fine-tuning task, with the pre-trained model downloaded from timm. The learning rate set to $0.01$ and a batch size of $128$. Our simulation leverages two image classification datasets: (1) the CIFAR-10 dataset \cite{krizhevsky2009learning}, containing colored images categorized into 10 classes. (2) the CIFAR-100 dataset \cite{krizhevsky2009learning}, containing colored images categorized into 100 classes. Notably, data distribution at vehicles is non-IID, which widely exists in practical systems. We realize data heterogeneity using Dirichlet distribution with parameter $\alpha$, where lower $\alpha$ results into more heterogeneous data partitions.
        
        \subsubsection{Experimental Setting for Communication and Computation Model}  
        The considered distributed network consists of $50$ mobile users randomly distributed within a range of 50 to 1000 m. The maximum transmission power for the devices is set between 0.5 and 2 W. The maximum GPU frequencies for the devices and the server are selected from the intervals $[1.0, 1.5]$ GHz and 3 GHz, respectively. There are $M = 20$ subcarriers available for assignment, with each subchannel having a bandwidth of 20 MHz. We assume that the mobile devices are equipped with GPUs, such as the Apple A15, which features 4 to 6 cores. Accordingly, the number of GPU cores on the device side, denoted as $C_n$, is selected from the range $[4, 6]$. The number of FLOPs per cycle per core, $D_n$, is set to 1 for all devices. The edge server is assumed to be equipped with advanced GPUs, such as the NVIDIA Tesla T4 and Tesla V100. Therefore, the number of GPU cores on the server side, denoted as $C_s$, is randomly assigned values from the interval $[2560, 5120]$, and the number of FLOPs per cycle per core, $D_s$, is chosen between 1 and 2.
        \begin{figure*}[t]
            \centering
            \subfloat[{Computation Time}]{
                \includegraphics[width=0.30\textwidth]{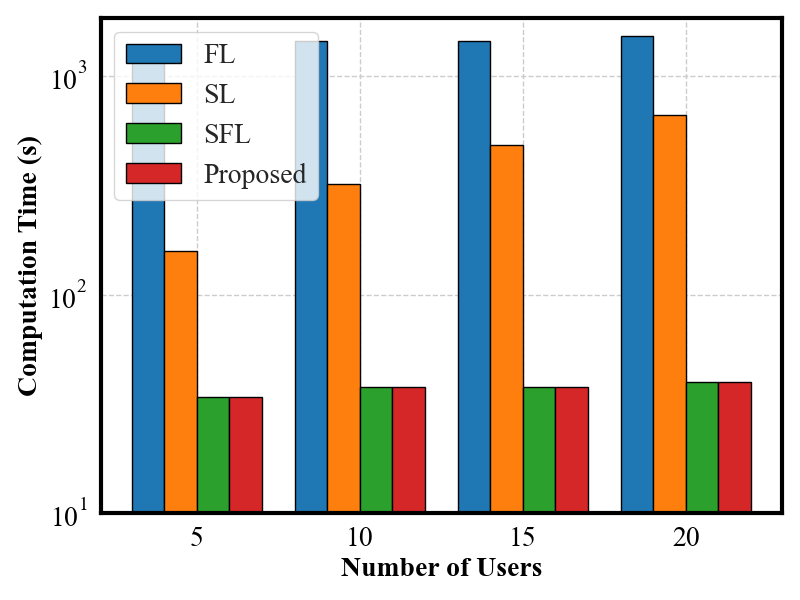}
                \label{fig:resource7}
            }
            \subfloat[{Communication Time}]{
                \includegraphics[width=0.30\textwidth]{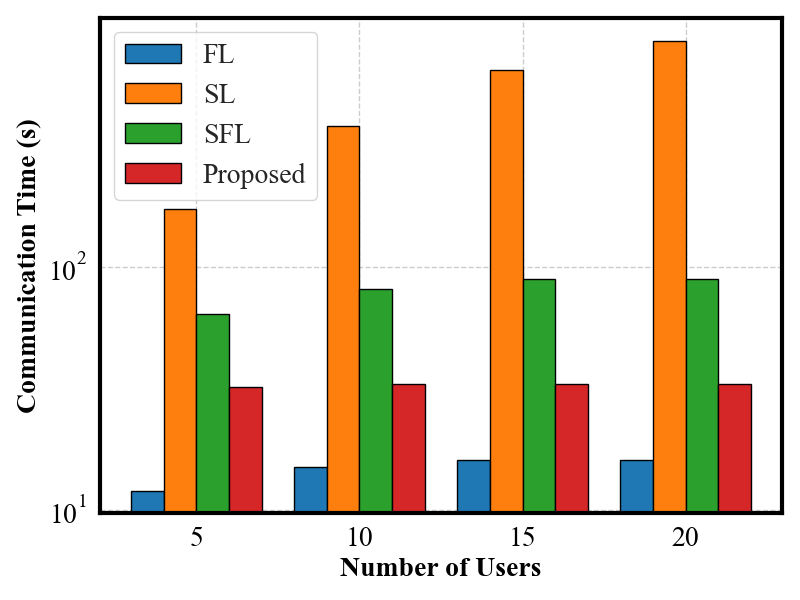}
                \label{fig:resource9}
            }
            \subfloat[{Overall Time}]{
                \includegraphics[width=0.30\textwidth]{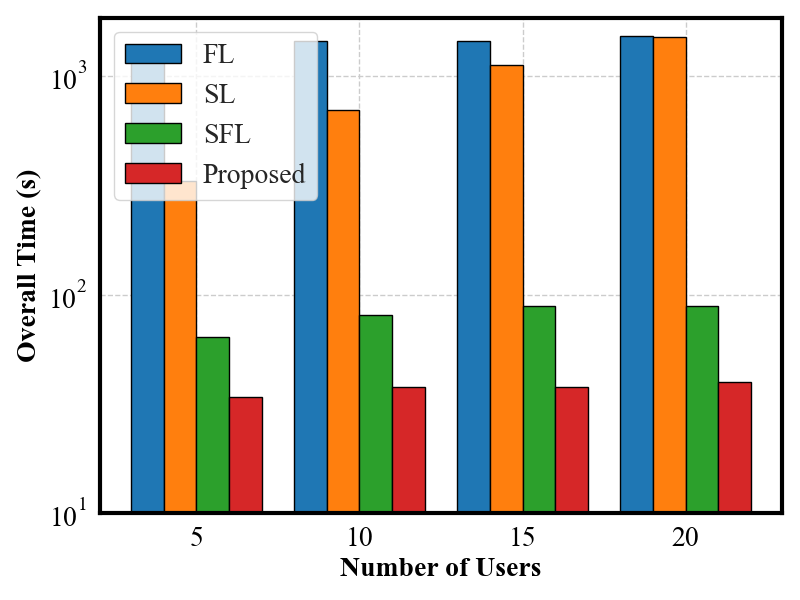}
                \label{fig:resource11}
            }
            \hfill
            \subfloat[{Computation Energy}]{
                \includegraphics[width=0.30\textwidth]{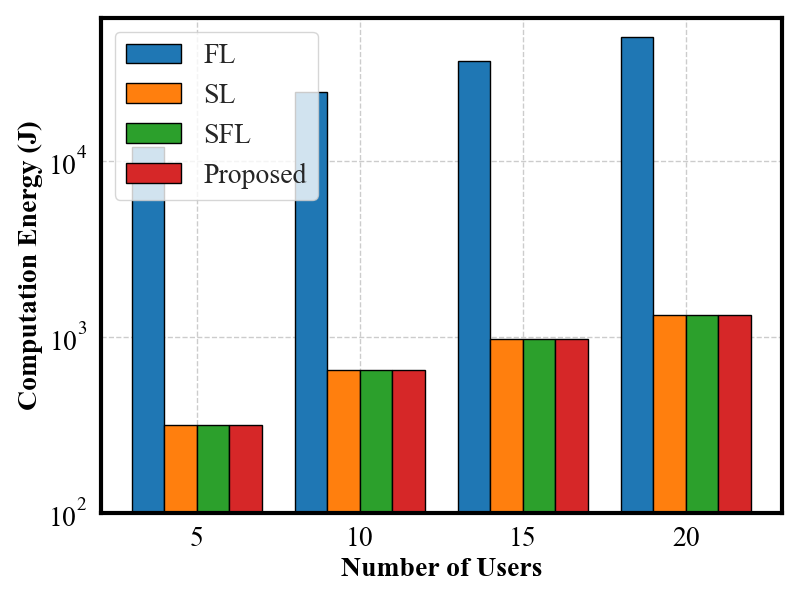}
                \label{fig:resource8}
            }
            \subfloat[{Communication Energy}]{
                \includegraphics[width=0.30\textwidth]{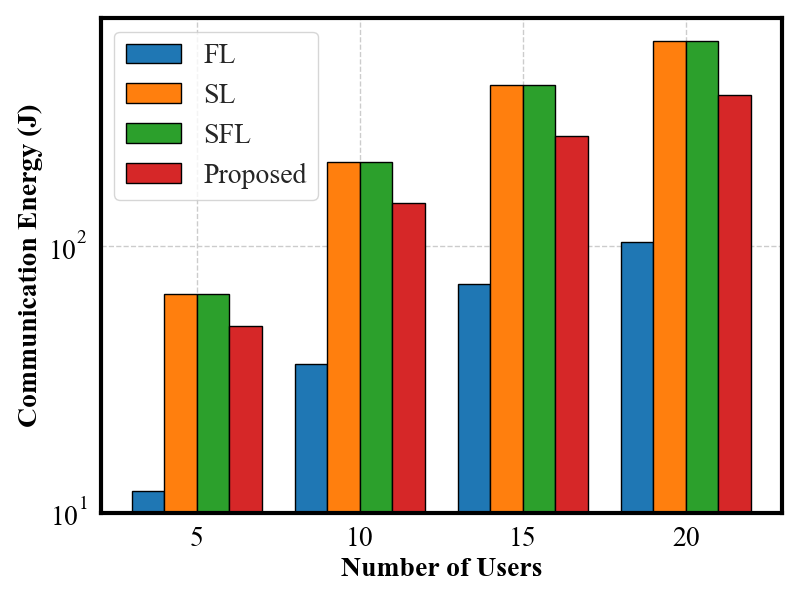}
                \label{fig:resource10}
            }
            \subfloat[{Overall Energy}]{
                \includegraphics[width=0.30\textwidth]{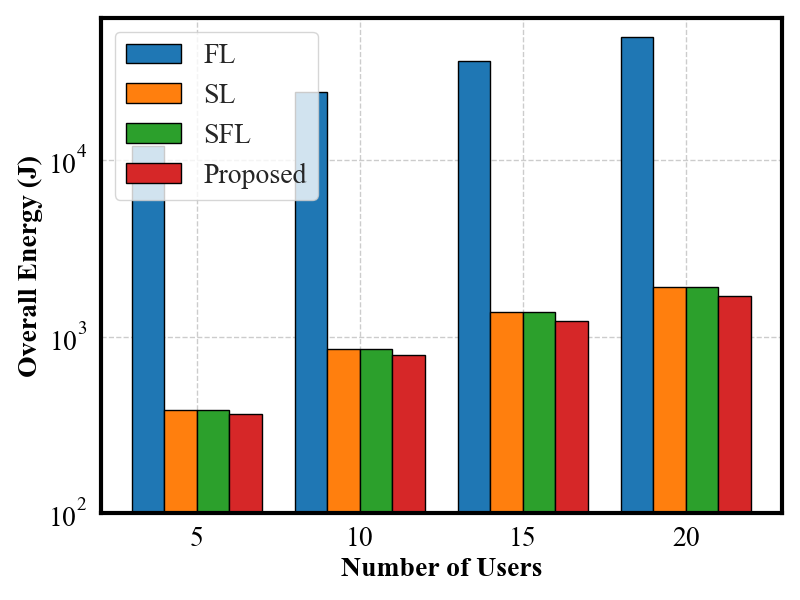}
                \label{fig:resource12}
            }
            \caption{Time and Energy Consumption}
            \label{fig:resource_allocation2}
        \end{figure*}     
    \subsection{Efficiency Evaluation of Resource Management}
    Fig. \ref{fig:resource_allocation1}\subref{fig:resource1}, \ref{fig:resource_allocation1}\subref{fig:resource2}, and  \ref{fig:resource_allocation1}\subref{fig:resource3} illustrate the convergence analysis of the BCD iteration method proposed in this study at different stages. Fig.  \ref{fig:resource_allocation1}\subref{fig:resource1} specifically shows the variation in average energy consumption across different BCD iterations under random $T_{max}$ constraints, after solving $\mathcal{SUBP}$1 and $\mathcal{SUBP}$2. The results indicate that a larger time constraint $T_{max}$ expands the optimization space, which in turn helps to reduce energy consumption more effectively. After solving the bandwidth allocation problem $\mathcal{SUBP}$3, the average energy consumption decreases significantly compared to $\mathcal{SUBP}$2 as shown in Fig. \ref{fig:resource_allocation1}\subref{fig:resource2}. Notably, when the time constraint is $T_{max} = 50s$, the number of devices meeting the constraints is fewer than the number of subchannels, resulting in higher average energy consumption. However, when $T_{max}$ is increased to $55s$ or $60s$, more devices are able to participate in the bandwidth allocation, thus lowering the average energy consumption. According to our objective function (\ref{eq-objective-function}), which aims to maximize system efficiency, Fig. \ref{fig:resource_allocation1}\subref{fig:resource3} shows that as the time constraint $T_{max}$ increases, devices can achieve higher quantization bits while maintaining similar energy expenditure, ultimately enhancing system efficiency. Therefore, a higher $T_{max}$ increases devices participation and quantization bits, optimizing the overall system performance.\par
    Fig. \ref{fig:resource_allocation1}\subref{fig:resource4}, Fig. \ref{fig:resource_allocation1}\subref{fig:resource5} and Fig. \ref{fig:resource_allocation1}\subref{fig:resource6} shows the average energy, objective value and average quantization vary as the time constraint $T_{max}$ changes in our proposed and three different resource allocation methods:
        \begin{itemize}
            \item Equal uplink transmit power (EP): With the uplink power fixed ($1.5$W), we focus solely on optimizing the wireless resource allocation and quantization management.
            \item Radom bandwidth allocation (RB): At each round, randomly selected devices are assigned subchannels with optimal power control and quantization management.
            \item No quantization management (NQ): The scheme is without quantization management ($32$bits) while with optimal bandwidth allocation and power control. 
        \end{itemize}
    Fig. \ref{fig:resource_allocation1}\subref{fig:resource4} shows that our proposed method achieves the lowest energy consumption. Comparing the NQ with the proposed methods, we observe that when $T_{max}$ is less than $60$s, fewer devices can meet the time constraints, leading to higher energy consumption in NQ. This suggests that the importance of quantization management increases, especially when time constraints are tight. The uplink power in EP is consistently set to $1.5$W, which results in higher energy consumption. While a higher uplink power enables more devices to meet the time constraint, the impact of quantization and resource allocation diminishes. Fig. \ref{fig:resource_allocation1}\subref{fig:resource5} demonstrates that our proposed method achieves the highest objective efficiency value. Comparing the NQ with the proposed methods, we find that quantization management performs well, particularly when time constraints are stringent. However, when the uplink power in EP is high, the effect of quantization management and bandwidth allocation is less significant. Finally, Fig. \ref{fig:resource_allocation1}\subref{fig:resource6} illustrates the average quantization bits. As shown, with an increase in $T_{max}$, the proposed method results in higher quantization bits.
    \subsection{Time-Energy Evaluation of SFLAM}
    Fig. \ref{fig:resource_allocation2} shows the time and energy consumption in one training round as the number of devices changes in our proposed and three different distributed training framework:
        \begin{itemize}
            \item Federated learning (FL): Devices train local models and aggregate updates into a global model, assuming devices have sufficient memory.
            \item Split learning (SL): Devices and the server sequentially train different model parts, collaborating at each step.
            \item Split federated learning (SFL): Devices and the server train model parts in parallel, aggregating both sides.
        \end{itemize}
            \begin{figure*}[htbp]
        \centering
        \subfloat[CIFAR-10 Dir(0.1)]{
            \includegraphics[width=0.30\textwidth]{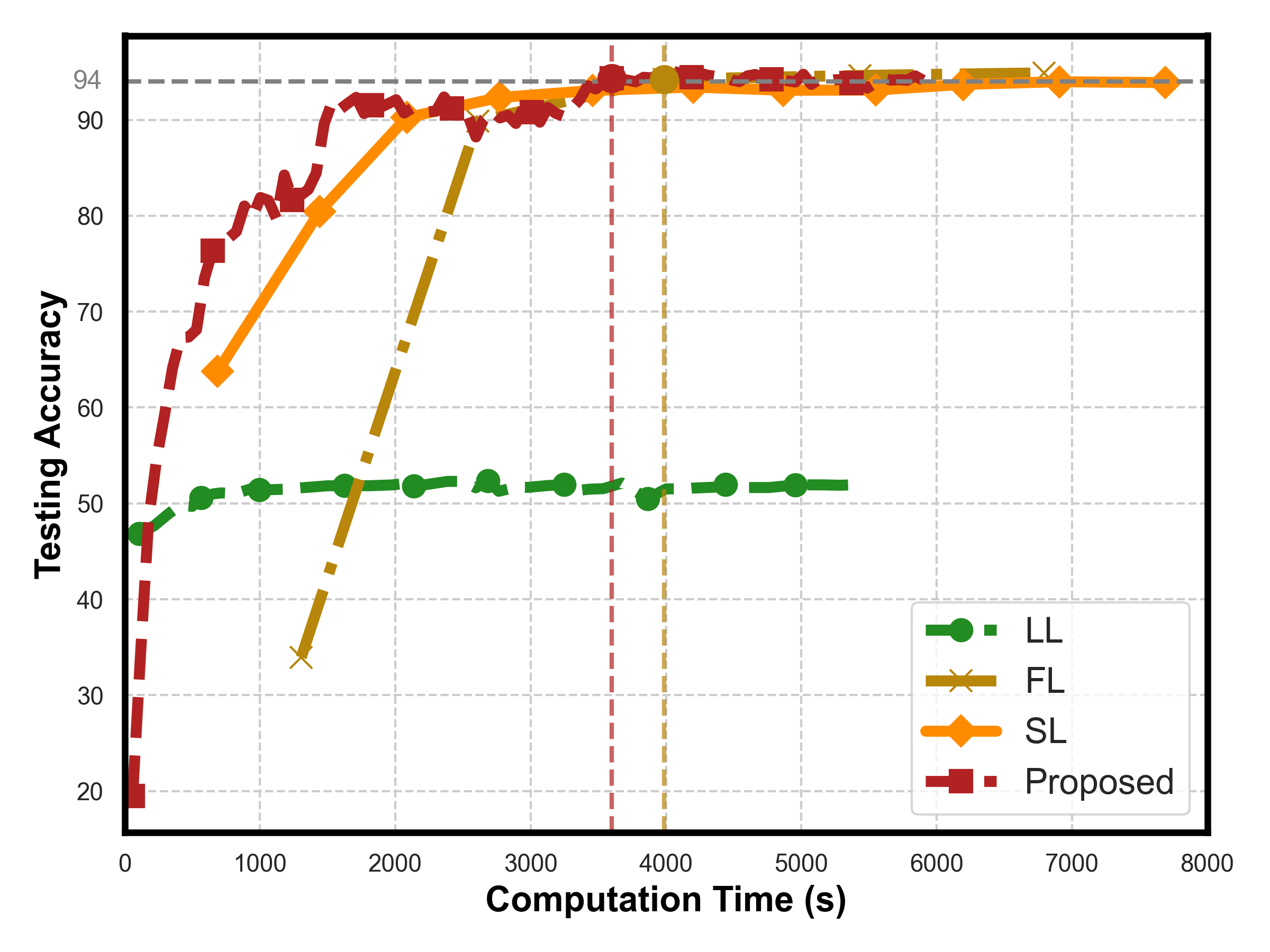}
            \label{fig:resource13}
        }
        \subfloat[CIFAR-10 Dir(0.3)]{
            \includegraphics[width=0.30\textwidth]{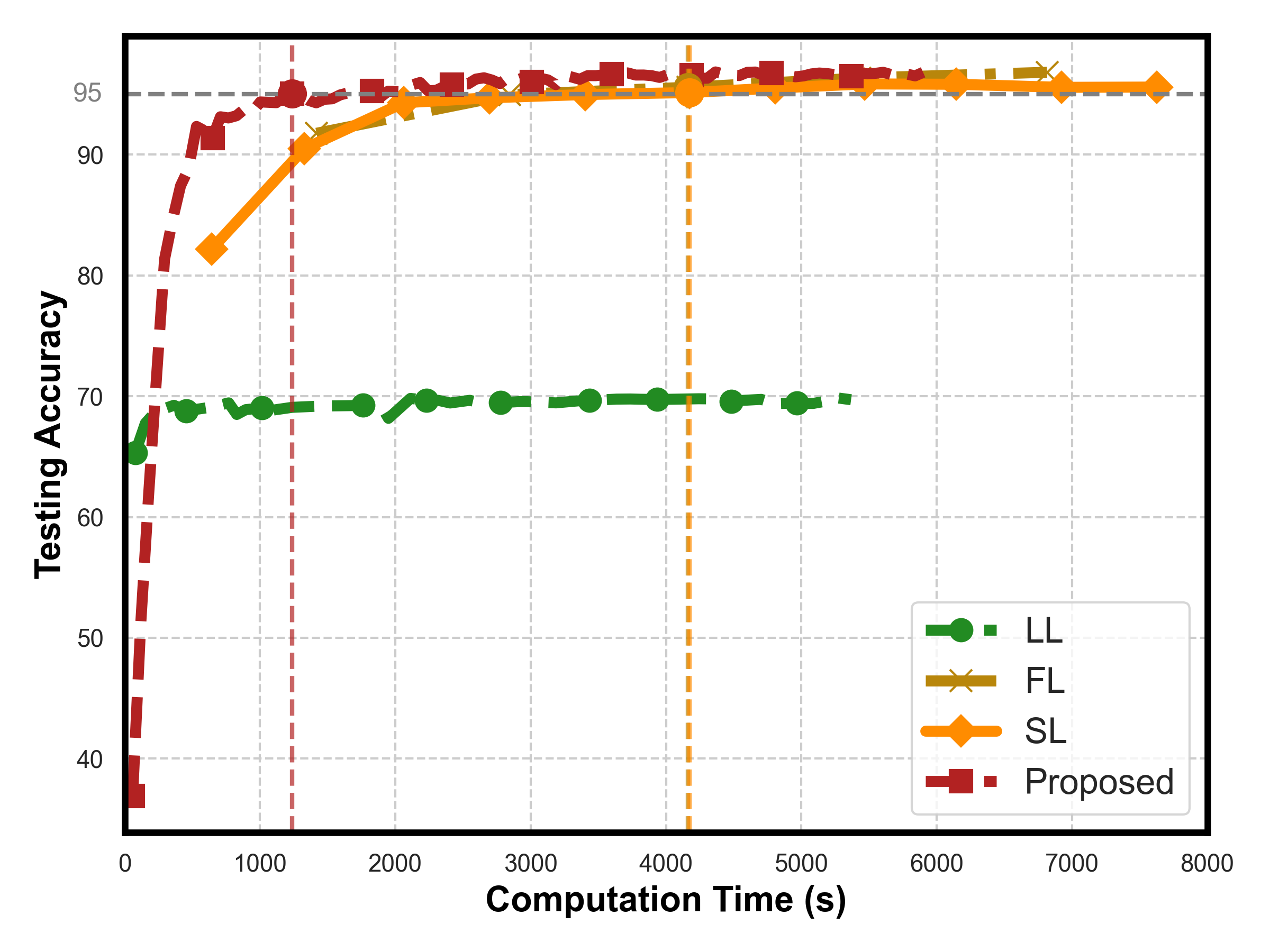}
            \label{fig:resource15}
        }
        \subfloat[CIFAR-10 Dir(0.5)]{
            \includegraphics[width=0.30\textwidth]{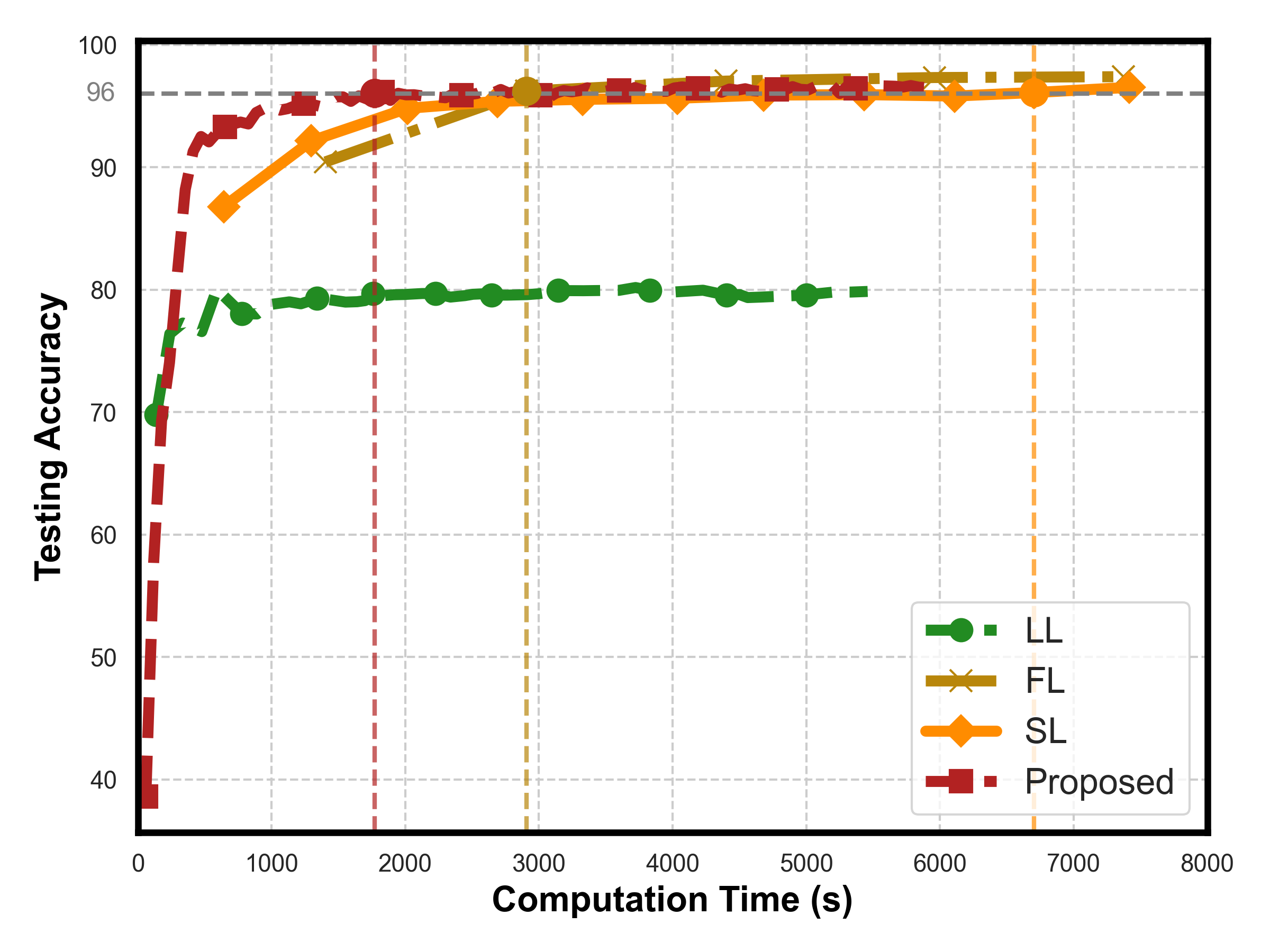}
            \label{fig:resource17}
        }
        \hfill
        \subfloat[CIFAR-100 Dir(0.1)]{
            \includegraphics[width=0.30\textwidth]{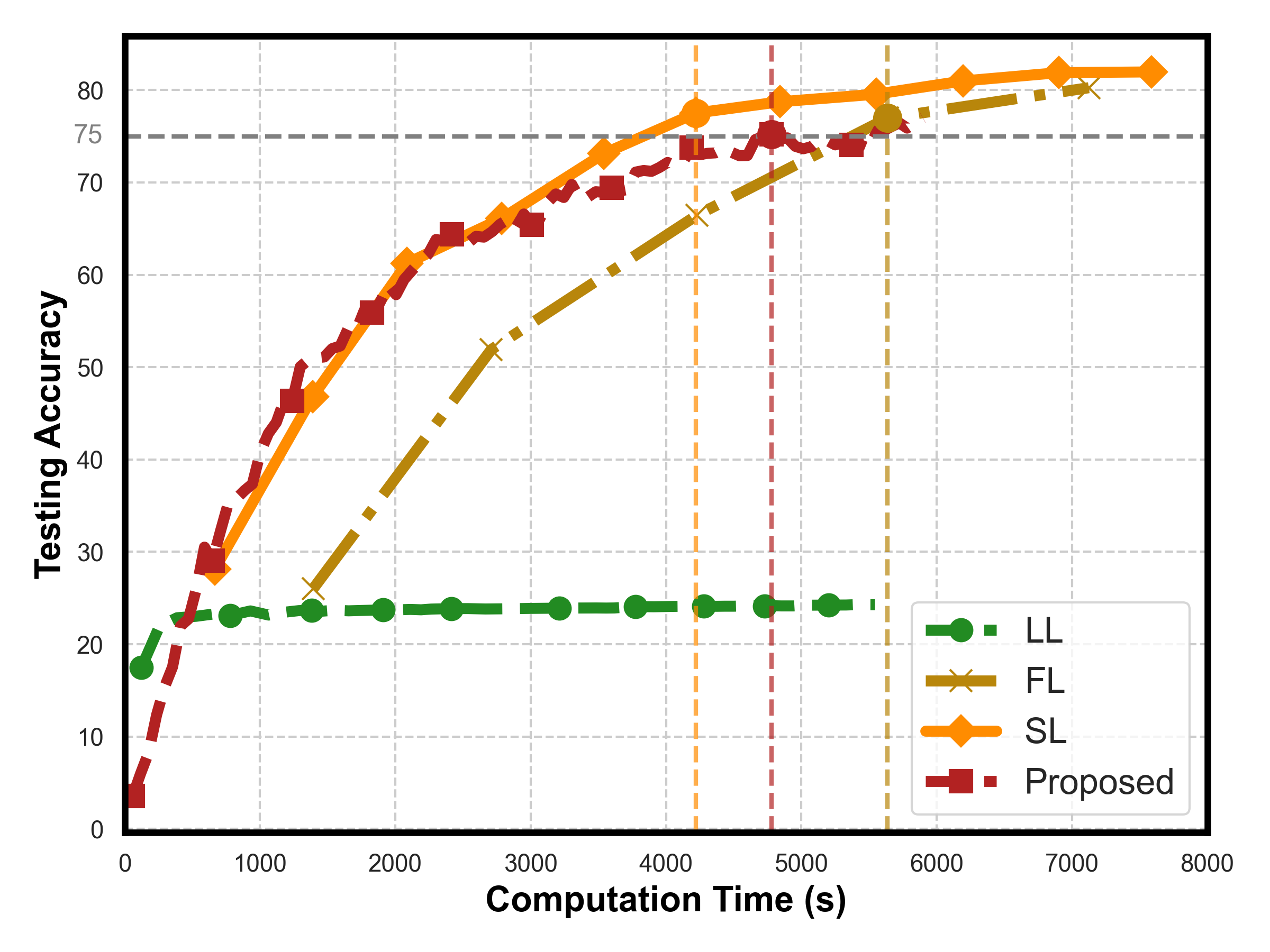}
            \label{fig:resource14}
        }        
        \subfloat[CIFAR-100 Dir(0.3)]{
            \includegraphics[width=0.30\textwidth]{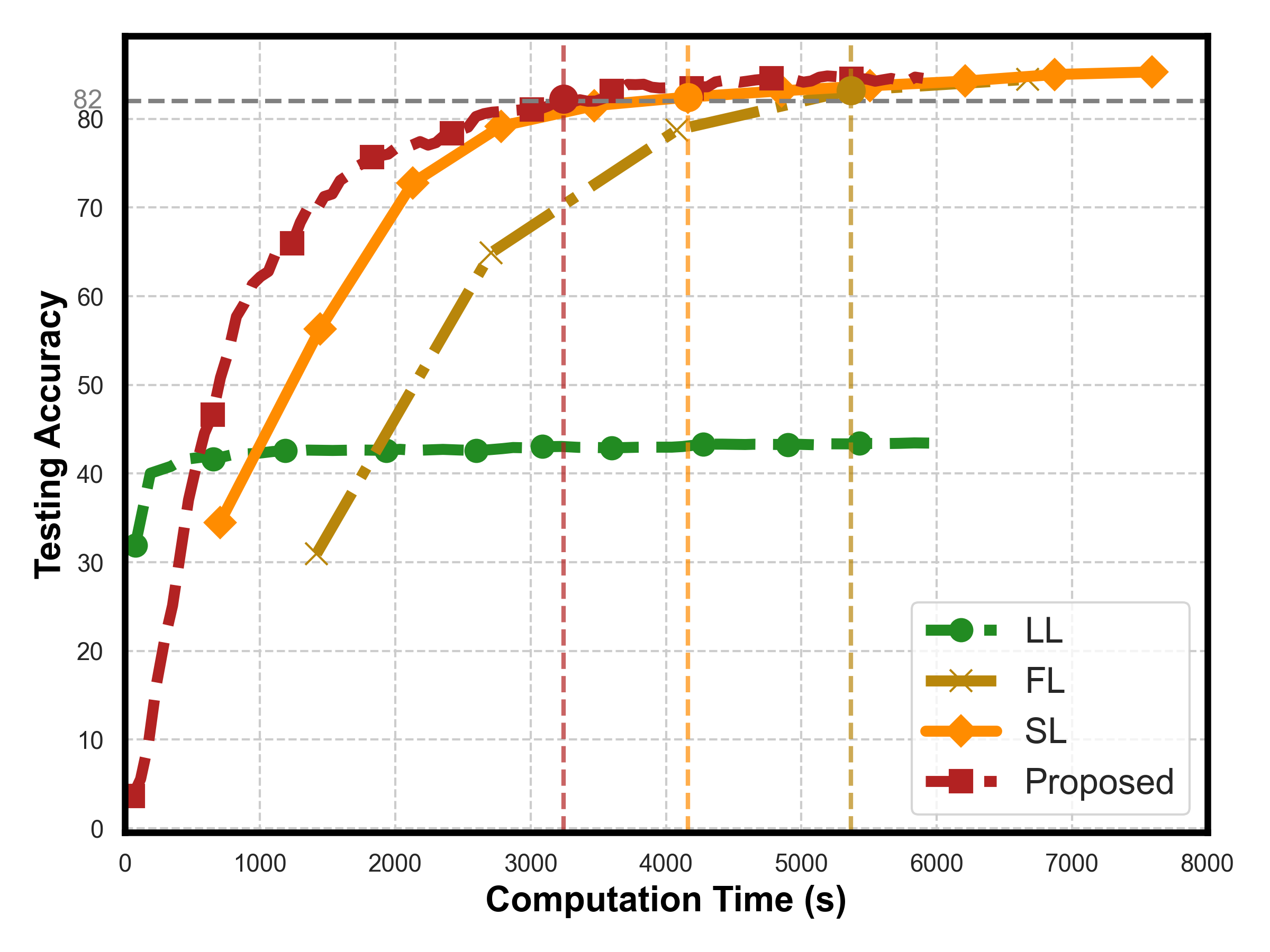}
            \label{fig:resource16}
        }
        \subfloat[CIFAR-100 Dir(0.5)]{
            \includegraphics[width=0.30\textwidth]{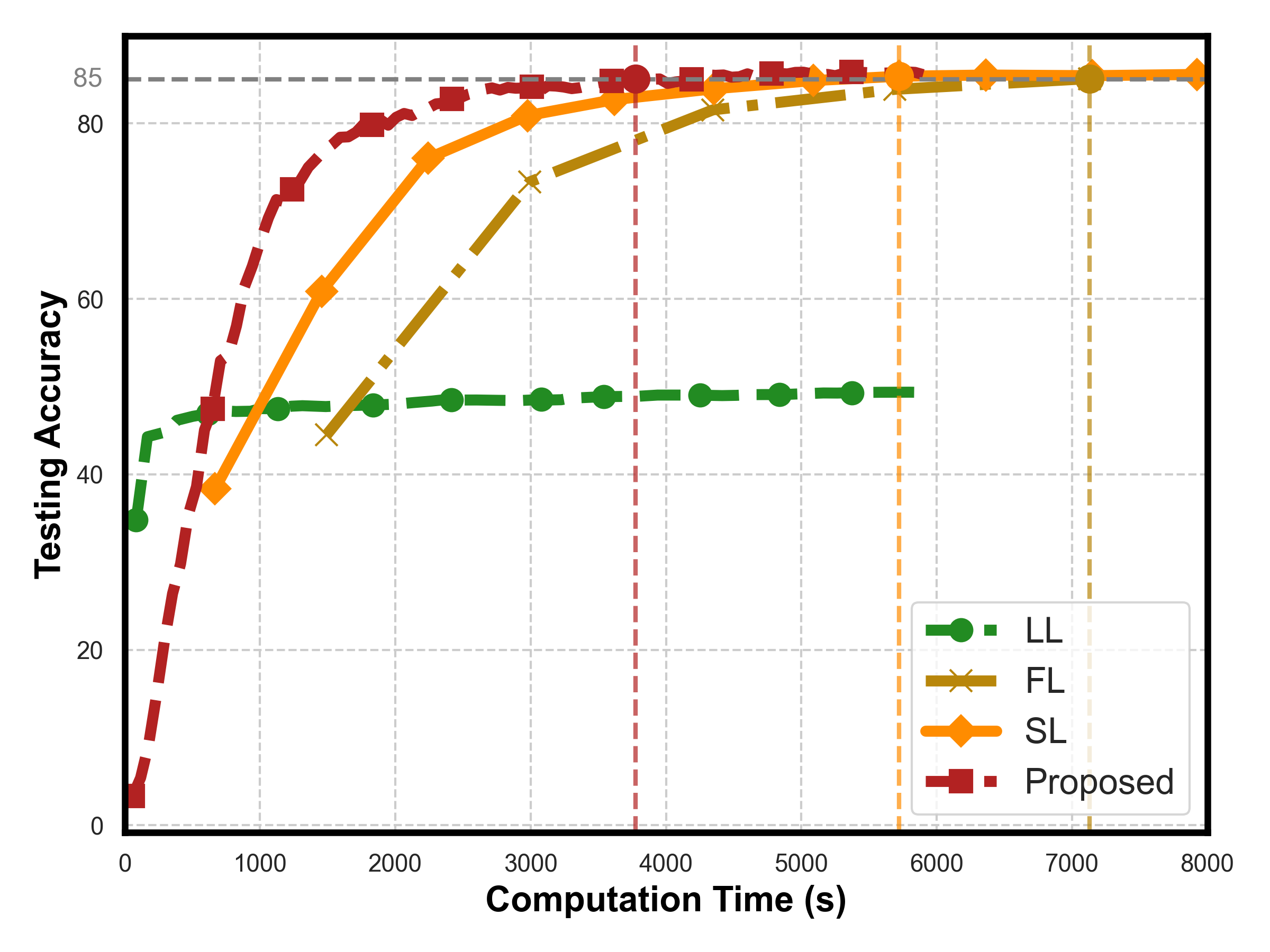}
            \label{fig:resource18}
        }
        \caption{Testing accuracy on Vit-Base/32 Model}
        \label{fig:acc_compare3}
    \end{figure*}   
    Fig. \ref{fig:resource_allocation2}\subref{fig:resource7} shows the computation time required for one training round using different methods. When devices possess adequate memory for the global model in FL, FL experiences the most significant latency due to the large Language Adaptation Model (LAM) and the limited computational capacity of the devices. In contrast, Split Learning (SL) achieves a reduction in device-side computation time while its sequential nature leads to increased overall delays. Due to parallel training on devices, SFL and proposed exhibit the shortest computation time. Fig. \ref{fig:resource_allocation2}\subref{fig:resource8} shows the computation energy consumption required for one training round on devices using different methods. FL has the highest energy consumption because the entire model is trained on the devices, while other frameworks consume less energy since only a portion of the model is trained on the devices.\par
    Fig. \ref{fig:resource_allocation2}\subref{fig:resource9} and \ref{fig:resource_allocation2}\subref{fig:resource10} show the communication time and energy required for one training round with different methods, respectively. FL incurs the least communication time and energy, while SL results in the longest due to its sequential nature. Compared to FL, SFL still demands more communication time and energy because it involves the additional transmission of intermediate parameters, such as activations and gradients, which are positively correlated with the amount of data. By incorporating quantization management, the proposed method significantly reduces both communication time and energy.\par
    Fig. \ref{fig:resource_allocation2}\subref{fig:resource11} and \ref{fig:resource_allocation2}\subref{fig:resource12} show the total time and energy consumption, respectively. FL consumes the most time, while the proposed method requires the least time. Although SL and SFL have longer communication times than FL, the overall time of FL is still higher, indicating that SL and SFL trade communication time for computation time. Through quantization management, the proposed method achieves the lowest time and energy consumption, demonstrating the effectiveness of the resource management strategy.

    \subsection{Performance Evaluation of SFLAM Fine-tuning}
    Fig. \ref{fig:acc_compare3} illustrates the test accuracy of the ViT-Base/32 model under various Dirichlet data distributions. The results demonstrate that the proposed SFLAM method achieves the fastest convergence. While SL also performs well, its sequential training process leads to considerable time delays. FL consistently yields the lowest accuracy, as training the entire model on terminal devices imposes a heavy computational burden, leading to prolonged training time. Local Learning (LL), where devices train models on their local datasets, results in the worst accuracy due to data heterogeneity. The proposed method achieves test accuracies of $94\%$, $95\%$, and $96\%$ on the CIFAR-10 dataset under Dir(0.1), Dir(0.3), and Dir(0.5), respectively. For CIFAR-100, the test accuracies under Dir(0.1), Dir(0.3), and Dir(0.5) are $75\%$, $82\%$, and $85\%$, respectively. As the Dirichlet parameter $\alpha$ increases, the test accuracy improves. The SFL-based framework still faces challenges associated with non-IID data, though these challenges are less pronounced compared to FL. In cases of complex and heterogeneous datasets, as shown in Fig. \ref{fig:acc_compare3}\subref{fig:resource14}, SL outperforms the proposed method, suggesting that SL has a strong advantage in mitigating data heterogeneity.

\section{Conclusion}
    \label{section8}
    This paper introduces Quantized Split Federated Fine-Tuning Large AI Model (SFLAM), an innovative framework tailored for deploying large AI models in resource-constrained edge environments. SFLAM strategically partitions model training between edge devices and a central server, incorporating quantization management, power control, and optimized bandwidth allocation. This approach significantly enhances training performance while minimizing energy consumption and communication latency. Both theoretical analysis and simulation results demonstrate SFLAM's superior learning efficiency and system scalability compared to existing methods, particularly in resource-limited settings. Future research can further explore the applicability of SFLAM in more complex scenarios and optimize its design under varying network conditions to facilitate broader real-world deployment.\par
\appendix
\subsection{Proof of Lemma \ref{Lemma-2}}
\label{proof-lemma2}
    We take the expectation over the random data distribution variable $\xi_n$ and also over the quantization error $\mathcal{Q}$:
	\begin{align}
		&\mathbb{E}_{\xi_n \sim \mathcal{D}_n} \mathbb{E}_Q [\nabla L_n(\boldsymbol{\omega}, \mathcal{Q}(\mathcal{A}_n(\xi_n)))] \nonumber \\
	\stackrel{(a)}{=}& \mathbb{E}_{\xi_n \sim \mathcal{D}_n}[\nabla L_n(\boldsymbol{\omega}, \mathcal{A}_n(\xi_n))]\nonumber \\
        &+ \mathbb{E}_{\xi_n \sim \mathcal{D}_n} \mathbb{E}_Q[\nabla^2 L_n(\boldsymbol{\omega}, \mathcal{A}_n(\xi_n))(\mathcal{Q}(\mathcal{A}_n(\xi_n)) - \mathcal{A}_n(\xi_n))],\nonumber \\
        &+\mathcal{O}(\|\mathcal{Q}(\mathcal{A}_n(\xi_n))-\mathcal{A}_n(\xi_n)\|^2),\nonumber\\
	=& \nabla L_n(\boldsymbol{\omega}, \mathcal{A}_n) \nonumber\\
        &+ \mathbb{E}_{\xi_n \sim \mathcal{D}_n} \mathbb{E}_Q \left[ \nabla^2 L_n(\boldsymbol{\omega}, \mathcal{A}_n(\xi_n))(\mathcal{Q}(\mathcal{A}_n(\xi_n)) - \mathcal{A}_n(\xi_n)) \right].\nonumber\\
        \stackrel{(b)}{\leq}& \nabla L_n(\boldsymbol{\omega}, \mathcal{A}_n)  + \mathcal{L} \cdot \| \mathcal{Q}(\mathcal{A}_n) - \mathcal{A}_n \|,\label{lemma2-1}
	\end{align}
	where the equality (a) is due to expand the gradient of the loss function $ L_n $ using a first-order Taylor expansion at $\mathcal{A}_n $, $\nabla^2 L_n(\boldsymbol{\omega}, \mathcal{A}_n) $ is the Hessian of $L_n$ with respect to $ \mathcal{A}_n$ and $ \nabla L_n(\boldsymbol{\omega},\mathcal{A}_n) $ represents the gradient of $L_n$ at the point $ \mathcal{A}_n $. where inequality (b) invoke the assumption \ref{assumption6} of $ \mathcal{L} $-smoothness for the second-order derivatives of the loss function, which states that $ \| \nabla^2 L_n(\boldsymbol{\omega}, \mathcal{A}_n) \| \leq \mathcal{L} $. This implies that the gradient variation due to the quantization error is bounded by:
	\begin{equation}
		\| \nabla^2 L_n(\boldsymbol{\omega}, \mathcal{A}_n)(\mathcal{Q}(\mathcal{A}_n) - \mathcal{A}_n) \| \leq \mathcal{L} \cdot \| \mathcal{Q}(\mathcal{A}_n) - \mathcal{A}_n \|. \nonumber
		\label{lemma2-quanerror}
	\end{equation}
	Substituting the bound into the previous expression gives the variance of the stochastic gradients:
	\begin{align}
		&\mathbb{E}_{\xi_n \sim \mathcal{D}_n} \mathbb{E}_Q \left[ \| \nabla L_n(\boldsymbol{\omega}, \mathcal{Q}(\mathcal{A}_n(\xi_n))) - \nabla L_n(\boldsymbol{\omega}) \|_2^2 \right] \notag \\
		=& \mathbb{E}_{\xi_n \sim \mathcal{D}_n} \mathbb{E}_Q \Big[ \| \nabla L_n(\boldsymbol{\omega}, \mathcal{Q}(\mathcal{A}_n(\xi_n))) - \nabla L_n(\boldsymbol{\omega}, \xi_n) \notag \\
		&+ \nabla L_n(\boldsymbol{\omega}, \xi_n) - \nabla L_n(\boldsymbol{\omega}) \|_2^2 \Big] \notag \\
		\stackrel{(a)}{=}& \mathbb{E}_{\xi_n \sim \mathcal{D}_n} \mathbb{E}_Q \left[ \| \nabla L_n(\boldsymbol{\omega}, \mathcal{Q}(\mathcal{A}_n(\xi_n))) 
		- \nabla L_n(\boldsymbol{\omega}, \xi_n) \|_2^2 \right] \notag \\
		&+ \mathbb{E}_{\xi_n \sim \mathcal{D}_n} \left[ \|  \nabla L_n(\boldsymbol{\omega})-\nabla L_n(\boldsymbol{\omega}, \xi_n) \|_2^2 \right] \notag \\
		&+ 2 \cdot \mathbb{E}_{\xi_n \sim \mathcal{D}_n} \mathbb{E}_Q \Big[ \big\langle \nabla L_n(\boldsymbol{\omega}, \nabla L_n(\boldsymbol{\omega}, \xi_n) - \nabla L_n(\boldsymbol{\omega}, \xi_n) ), \notag \\
		&\quad \nabla L_n(\boldsymbol{\omega}, \xi_n) - \nabla L_n(\boldsymbol{\omega}) \big\rangle \Big].\notag\\
		\stackrel{(b)}{=}&\mathbb{E}_{\xi_n \sim \mathcal{D}_n} \mathbb{E}_Q \left[ \| \nabla L_n(\boldsymbol{\omega}, \mathcal{Q}(\mathcal{A}_n(\xi_n))) 
		- \nabla L_n(\boldsymbol{\omega}, \xi_n) \|_2^2 \right] \notag\\
		&+ \mathbb{E}_{\xi_n \sim \mathcal{D}_n} \left[ \| \nabla L_n(\boldsymbol{\omega}, \xi_n) - \nabla L_n(\boldsymbol{\omega}) \|_2^2 \right] \notag \\
		\stackrel{(c)}{=}&\mathcal{L}^2\cdot \delta A^2 +\sigma_n^2,\nonumber
	\end{align}
	where the equality (a) is due to $\|a-b\|_2^2=\|a\|_2^2+2 \big\langle a,b \big\rangle + \|b\|_2^2$. The equality (b) is due to $\mathbb{E}_{\xi_n \sim \mathcal{D}_n}\left[\nabla L_n(\boldsymbol{\omega}, \xi_n) - \nabla L_n(\boldsymbol{\omega})\right]=0$, so that $2\langle a,b \rangle=0$. The  equality (c) is obtained from Lemma \ref{lemma1}, Assumption \ref{assump3} and (\ref{lemma2-1}), and $A^2=\mathcal{L}^2\cdot \delta \|\mathcal{A}_n(\xi_n)\|^2$. Here, $\sigma_n^2$ represents the variance of the gradient due to randomness in $\xi_n$, and $\mathcal{L}^2\cdot \delta A^2$ is the additional variance introduced by the quantization error.

\subsection{Proof of Lemma \ref{Lemma-3}}
    	\label{proof-lemma3}
	Taking expectations over both $\xi_n \sim \mathcal{D}_n$ and the quantization error $\mathcal{Q}$, we obtain the following expression: 
\begin{align}
    &\mathbb{E}_{\xi_n \sim \mathcal{D}_n} \mathbb{E}_Q \left[ \| g_n(\boldsymbol{\omega}, \mathcal{Q}(\mathcal{A}_n(\xi_n))) \|_2^2 \right] \nonumber\\
    =&\mathbb{E}_{\xi_n \sim \mathcal{D}_n} \mathbb{E}_Q \left[ \| \nabla L_n(\boldsymbol{\omega}, \mathcal{Q}(\mathcal{A}_n(\xi_n))) \|_2^2 \right] \notag \\
    \stackrel{(a)}{=}&\mathbb{E}_{\xi_n \sim \mathcal{D}_n} \mathbb{E}_Q \left[ \| \nabla L_n(\boldsymbol{\omega}, \mathcal{A}_n(\xi_n)) \right. \nonumber\\
    &\quad + \nabla^2 L_n(\boldsymbol{\omega}, \mathcal{A}_n(\xi_n)) (\mathcal{Q}(\mathcal{A}_n(\xi_n)) - \mathcal{A}_n(\xi_n)) \|_2^2 \Big] \nonumber\\
    \stackrel{(b)}{=}& \mathbb{E}_{\xi_n \sim \mathcal{D}_n} \left[ \|\nabla L_n(\boldsymbol{\omega}, \mathcal{A}_n(\xi_n)) \|_2^2 \right] \notag \\
    &+ \mathbb{E}_{\xi_n \sim \mathcal{D}_n} \mathbb{E}_Q \left[ \| \nabla^2 L_n(\boldsymbol{\omega}, \mathcal{A}_n(\xi_n)) (\mathcal{Q}(\mathcal{A}_n(\xi_n)) - \mathcal{A}_n(\xi_n)) \|_2^2 \right] \notag \\
    &+ \mathbb{E}_{\xi_n \sim \mathcal{D}_n} \mathbb{E}_Q \left[ 2\big\langle \nabla L_n(\boldsymbol{\omega}, \mathcal{A}_n(\xi_n)),\nabla^2 L_n(\boldsymbol{\omega}, \mathcal{A}_n(\xi_n)) \right. \nonumber\\
    &\quad \left.  (\mathcal{Q}(\mathcal{A}_n(\xi_n)) - \mathcal{A}_n(\xi_n)) \big\rangle \right], \nonumber \\
    \stackrel{(c)}{=}& \mathbb{E}_{\xi_n \sim \mathcal{D}_n} \left[ \| \nabla L_n(\boldsymbol{\omega}, \mathcal{A}_n(\xi_n)) \|_2^2 \right] \nonumber \\
    &+ \mathbb{E}_{\xi_n \sim \mathcal{D}_n} \mathbb{E}_Q \left[ \| \nabla^2 L_n(\boldsymbol{\omega}, \mathcal{A}_n(\xi_n)) (\mathcal{Q}(\mathcal{A}_n) - \mathcal{A}_n(\xi_n)) \|_2^2 \right]. \nonumber
\end{align}
	where the equality (a) follows from the Taylor expansion, (b) comes from applying $\|a-b\|_2^2=\|a\|_2^2+2 \big\langle a,b \big\rangle + \|b\|_2^2$. The equality (c) results from $\mathbb{E}_{\mathcal{Q}}\left[\mathcal{Q}(\mathcal{A}_n)\right]=\mathcal{A}_n$, so that $2\langle a,b \rangle=0$. As in Lemma \ref{Lemma-2}, we use the smoothness assumption $\| \nabla^2 L_n(\boldsymbol{\omega}, \mathcal{A}_n(\xi_n)) \| \leq \mathcal{L} $ and the quantization error $\mathbb{E}\left[\|\mathcal{Q}(\mathcal{A})-\mathcal{A}\|_2^2\right]\leq \delta \|\mathcal{A}\|_2^2$ in Lemma \ref{lemma1}, which gives:
	\begin{align}
		&\mathbb{E}_Q \left[ \| \nabla^2 L_n(\boldsymbol{\omega}, \mathcal{A}_n(\xi_n)) (\mathcal{Q}(\mathcal{A}_n) - \mathcal{A}_n) \|_2^2 \right] \nonumber\\
		\leq& \mathcal{L}^2 \cdot \mathbb{E}_Q \left[ \| \mathcal{Q}(\mathcal{A}_n) - \mathcal{A}_n \|^2 \right]\nonumber\\
		\leq& \mathcal{L}^2 \cdot \delta \|\mathcal{A}_n\|^2.\nonumber
	\end{align}
	Finally, we can now bound the expected norm of the stochastic gradients as:
	\begin{align}
		&\mathbb{E}_{\xi_n \sim \mathcal{D}_n} \mathbb{E}_Q \left[ \| g_n(\boldsymbol{\omega}, \mathcal{Q}(\mathcal{A}_n(\xi_n))) \|_2^2 \right] \nonumber\\
        =&\mathbb{E}_{\xi_n \sim \mathcal{D}_n} \left[ \| \nabla L_n(\boldsymbol{\omega}, \mathcal{A}_n(\xi_n)) \|_2^2 \right] \nonumber \\
        &+ \mathbb{E}_{\xi_n \sim \mathcal{D}_n} \mathbb{E}_Q \left[ \| \nabla^2 L_n(\boldsymbol{\omega}, \mathcal{A}_n(\xi_n)) (\mathcal{Q}(\mathcal{A}_n) - \mathcal{A}_n(\xi_n)) \|_2^2 \right]. \nonumber\\
		\leq& G^2 + \mathcal{L}^2 \cdot \delta \| \mathcal{A}_n \|^2.\nonumber
	\end{align}
	This concludes the proof for Lemma \ref{Lemma-3}.
	
	\subsection{Proof of Theorem \ref{theorem1}}
    \label{proof-theorem1}
	   We can derive the upper bound as shown in (\ref{eq-theorm}), where $\delta = \frac{1+\sqrt{2d-1}}{2}\frac{1}{2^q-1}$, $A = \|\mathcal{A}_n(\xi_n)\|^2$, and $ B = \frac{8SN}{\mu^2(\gamma + k)} \sum_{n=1}^N \rho_n^2 \left( 3 + \frac{1}{a_n} \right) + \frac{768S^2}{\mu^3(\gamma + k)(\gamma + 1)} \sum_{n=1}^N 3a_n$. The detailed steps are similar in \cite{han2024convergence}.
    \begin{figure*}
    \begin{align}
        \label{eq-theorm}
        &\mathbb{E}[\mathbf{L}(\boldsymbol{\omega}_k)] - \mathbf{L}(\boldsymbol{\omega}^*)\\
        \leq & \frac{S}{2} \left( \mathbb{E}[|\boldsymbol{\omega}_k^{\mathcal{S}} - \boldsymbol{\omega}^{\mathcal{S}}|^2] + \mathbb{E}[|\boldsymbol{\omega}_k^{\mathcal{C}} - \boldsymbol{\omega}^{\mathcal{C}*}|^2] \right)\nonumber\\
        \leq & \frac{8 S N \sum_{n=1}^N \rho_n^2\left(2 (\sigma_n^2+\mathcal{L}^2 \cdot \delta \| \mathcal{A}_n \|^2)+(G^2+\mathcal{L}^2 \cdot \delta \| \mathcal{A}_n \|^2)+\frac{G^2+\mathcal{L}^2 \cdot \delta \| \mathcal{A}_n \|^2}{a_n}\right)}{\mu^2(\gamma+T)} \nonumber\\
        &+\frac{768 S^2 \sum_{n=1}^N \rho_n\left(2 (\sigma_n^2+\mathcal{L}^2 \cdot \delta \| \mathcal{A}_n \|^2)+(G^2+\mathcal{L}^2 \cdot \delta \| \mathcal{A}_n \|^2)\right)}{\mu^3(\gamma+T)(\gamma+1)} + \frac{S(\gamma+1) \mathbb{E}\left[\left\|\boldsymbol{x}^0-\boldsymbol{x}^*\right\|^2\right]}{2(\gamma+T)}\nonumber\\
        \leq& \frac{8SN}{\mu^2(\gamma + k)} \left( \sum_{n=1}^N \rho_n^2 \left[ 2\sigma_n^2 + G^2 + \frac{G^2}{a_n} \right] \right) + \frac{768S^2}{\mu^3(\gamma + k)(\gamma + 1)} \left( \sum_{n=1}^N \rho_n \left[ 2\sigma_n^2 + G^2 \right] \right) + \frac{S(\gamma + 1)}{2(\gamma + K)} \left( \mathbb{E}[\|\boldsymbol{\omega}_0 - \boldsymbol{\omega}^{*}\|^2] \right) + \mathcal{L}^2 \cdot \delta A^2 B. \nonumber
    \end{align}
      {\noindent} \rule[-10pt]{18cm}{0.05em}
    \end{figure*}
	

	\bibliographystyle{IEEEtran}
	\bibliography{main.bib}

\end{document}